\documentclass[final,3p,times]{elsarticle}
\usepackage{amssymb}
\usepackage{amsthm}
\usepackage{amstext}
\usepackage{amsmath}
\usepackage{yhmath}

\newtheorem{definition}{Definition}

\newtheorem{theorem}{\bf Theorem}
\newtheorem{proposition}{Proposition}
\newtheorem{lemma}{\bf Lemma}

\newcommand{\argmin}{\operatornamewithlimits{arg \, min}}

\journal{Applied and Computational Harmonic Analysis}

\begin{document}

\begin{frontmatter}



\title{Optimal Learning with Anisotropic Gaussian SVMs}


\author[mymainaddress]{Hanyuan Hang\corref{mycorrespondingauthor}}
\cortext[mycorrespondingauthor]{Corresponding author}
\ead{hans2017@ruc.edu.cn}

\author[mysecondaryaddress]{Ingo Steinwart}

\address[mymainaddress]{Institute of Statistics and Big Data, Renmin University of China, Beijing 100872, China}
\address[mysecondaryaddress]{Institute for Stochastic and Applications, University of Stuttgart, Stuttgart 70569, Germany}

\begin{abstract}
This paper investigates the nonparametric regression problem using SVMs with anisotropic Gaussian RBF kernels. Under the assumption that the target functions are resided in certain anisotropic Besov spaces, we establish the almost optimal learning rates, more precisely, optimal up to some logarithmic factor, presented by the effective smoothness.
By taking the effective smoothness into consideration, our almost optimal learning rates are faster than those obtained with the underlying RKHSs being certain anisotropic Sobolev spaces.
Moreover, 
if the target function depends only on fewer dimensions,
faster learning rates can be further achieved.  
\end{abstract}


\end{frontmatter}


\section{Introduction} \label{sec::introduction}
Kernels and kernel methods have been gaining their prevalence as standard tools in machine learning. The essential idea in kernel methods lies in the transformation of the input data to certain high-dimensional feature space with certain nice computational properties preserved, which is the so-called \textit{kernel trick}. Literature has witnessed a growth of various kernel-based learning schemes and a flourish of studies, e.g.~\cite{ScSm01, SuVaDe02, StCh08}. Kernel methods are black-boxes in that the feature maps are not interpretable and one cannot know their explicit formulas, which actually brings computational convenience. However, it treats all features the same, which is usually not in line with the case in practice.

The great empirical success of deep learning in the past decade is attributed to the feature learning/ feature engineering brought by the inherent structure of multiple hidden layers \cite{Bengio09a, BeCoVi13a}.
However, feature engineering can never be ad-hoc for deep learning, and can actually play a role in the context of kernel methods.
In fact, it is shown that deep learning models are closely related to kernel methods and many of them can be interpreted by using kernel machines \cite{ChSa09a,AnRoTaPo15a}. Besides, several deep kernel methods have also been proposed by introducing \textit{deeper kernels} \cite{ChSa09a,AnRoTaPo15a}. 

Noticing the necessities of feature representation in the context of kernel methods, the present study aims at investigating learning problems with a special class of kernel functions, i.e., anisotropic kernels. The concept of \textit{anisotropic} here is termed as a generalization of the vanilla kernel class, namely, the isotropic kernels. Different from the isotropic ones, the specialty of anisotropic kernels lies in that their shape parameters entail feature-wise adjustments. Though that may involves more hyper-parameters, it may also improves the empirical performance. The validity of the anisotropic kernels can be demonstrated via a wide range of applications, especially the image processing problems, such as lung pattern classification \cite{ShSo04a} and forecasting  \cite{GaOrSaCaSaPo13a} via SVMs with anisotropic Gaussian RBF kernels. Literature has been focused on the feature selection capacity that anisotropic kernels bring, for example, \cite{Allen13a} presents an algorithm to minimize a feature-regularized loss function and thus achieving automatic feature selection based on the feature-weighted kernels, \cite{LiYaXi05a} proposes the so called Feature Vector Machine which can be easily extended for feature selection with non-linear models by introducing kernels defined on feature vectors, and
\cite{MaWeBa11a} introduces the kernel-penalized SVM that simultaneously selects relevant features during classifier construction by penalizing each feature’s use in the dual formulation of SVMs. 
To cope with the highdimensional nature of the input space, \cite{BrRoCrSe07a} proposes a density estimator with anisotropic kernels which are especially appropriate when the density concentrates on a low-dimensional subspace. 
Finally, \cite{ShZh12a} proposes a noise-robust edge detector which combines a small-scaled isotropic Gaussian kernel and large-scaled anisotropic Gaussian kernels to obtain edge maps of images.

Fully aware of the significance of the anisotropic kernels, we address the learning problem of non-parametric least squares regression with anisotropic Gaussian RBF kernel-based support vector machines. The learning goal is to find a function $f_D: X \to \mathbb{R}$ that is a good estimate of the unknown conditional mean $f^*(x) := \mathbb{E}(Y| \boldsymbol{x})$, $\boldsymbol{x} \in X$ given $n$ i.i.d.~observations  $D := ( (\boldsymbol{x}_1, y_1), \ldots, (\boldsymbol{x}_n, y_n) )$ of input/output pairs drawn from an unknown distribution $\mathrm{P}$ on $X \times Y$, where $Y \subset \mathbb{R}$. 
To be specific, we ought to find a function $f_D$ such that, for the loss function $L: Y \times \mathbb{R} \to [0, \infty)$, the risk 
\begin{align*}
\mathcal{R}_{L,\mathrm{P}} (f_{\mathrm{D}}) := \int_{X \times Y} L(y, f_{\mathrm{D}}(x)) \, d\mathrm{P}(x, y) 
\end{align*}
should be close to the Bayes risk
$\mathcal{R}_{L,\mathrm{P}}^*
:= \inf \bigl\{ \mathcal{R}_{L,\mathrm{P}}(f) \ | \ f : X \to \mathbb{R} \text{ measureable} \bigr\} $
with respect to $\mathrm{P}$ and $L$. A Bayes decision function is a function $f_{L,\mathrm{P}}^*$ satisfying 
$\mathcal{R}_{L,\mathrm{P}}(f_{L,\mathrm{P}}^*) = \mathcal{R}_{L,\mathrm{P}}^*$.

In this paper, we consider an anisotropic kernel-based regularized empirical risk minimizers, namely support vector machine (SVMs) with anisotropic kernel $k_{\boldsymbol{w}}$, which solve the regularized problem
\begin{align} \label{f_D,lambda,gamma}
	f_{\mathrm{D},\lambda,\boldsymbol{w}} 
	= \argmin_{f \in H_{\boldsymbol{w}}} 
	\frac{1}{n} \sum_{i=1}^{n} L(y_i,f(x_i)) 
	+ \lambda \|f\|_{H_{\boldsymbol{w}}}^2\ . 
\end{align}
Here, $\lambda > 0$ is a fixed real number and $H_{\boldsymbol{w}}$ is a reproducing kernel Hilbert space (RKHS) of $k_{\boldsymbol{w}}$ over $X$. 

The main contribution of this paper lies in the establishment of the almost optimal learning rates for nonparametric regression with anisotropic Gaussian SVMs, provided that the target functions are contained in some \textit{anisotropic Besov spaces}.
Recall that the overall smoothness of the commonly used isotropic kernels depends on the worst smoothness of all dimensions, which faces the dilemma where one poor smoothness along certain dimension may lead to unsatisfying convergence rates of the decision functions, even when smoothness along other dimensions is fairly good. Unlike the isotropic ones, the anisotropic kernels are more resistant to the poor smoothness of certain dimensions. To be specific, the overall smoothness is embodied by the \textit{effective smoothness}, or the \textit{exponent of global smoothness}, whose reciprocal is the mean of the reciprocals of smoothness of all dimensions. In this manner, poor smoothness along certain dimensions will not able to jeopardize the whole good one. Based on the effective smoothness, we manage to derive almost optimal learning rates which not only match the theoretical optimal ones for anisotropic kernels up to some logarithmic factor, but also in line with the published optimal learning rates derived by different algorithms. Moreover, when embed our results in cases of the isotropic one where we take all shape parameters as the same, our optimal learning rates still coincide with the theoretical optimal ones for isotropic kernels up to a logarithmic factor and are even better than the existing rates obtained via other methods.

Moreover, even though literature mainly concentrates on isotropic classes, the assumption of this isotropy might result in the loss of efficiency if the regression function actually belongs to an anisotropic class. In fact, this inefficiency is getting worser with the dimension getting higher. Therefore, assumption of anisotropy might serve as a more appropriate substitute. Besides, the anisotropy assumption also shows its advantages in confronting sparse regression functions where the learning rates will automatically depend on a small subset of the coordinates owing to the nature of effective smoothness. This phenomenon is also supported by theoretical analysis in this paper with even faster learning rates established.

The paper is organized as follows: Section \ref{sec::preliminaries} summarizes notations and preliminaries. We present the main results of the almost optimal learning rates of the anisotropic kernels for regression in Section \ref{sec::convergence_rates}. The error analysis is clearly illustrated in Section \ref{sec::error_analysis}. Detailed proofs of Sections \ref{sec::convergence_rates} and \ref{sec::error_analysis} are placed in Section \ref{sec::proofs}, for the sake of clarity. Finally, we conclude this paper in Section \ref{sec::DisCon}.

\section{Preliminaries} \label{sec::preliminaries}

Throughout this paper, we assume that $X\subset\mathbb{R}^d$ is a non-empty, open and bounded set such that its boundary $\partial X$ has Lebesgue measure $0$, $Y:=[-M,M]$ for some   $M>0$ and $\mathrm{P}$ is a probability measure on $X\times Y$ such that 
$\mathrm{P}_X$ on $X$ is absolutely continuous with respect to the Lebesgue measure on $X$.
Furthermore, we assume that the corresponding density of $\mathrm{P}_X$ is bounded away from $0$ and $\infty$. 
In what follows, we denote the closed unit ball of a Banach space $E$ by $B_E$, the $d$-dimensional Euclidean space $\ell_2^d$, we write $B_{\ell_2^d}$. 
For $s \in \mathbb{R}$, $\lfloor s \rfloor$ is the greatest
integer smaller or equal $s$ and $\lceil s \rceil$ is the smallest integer greater or equal $s$.
The tensor product $f \otimes g : X \times X \to \mathbb{R}$ of two
functions $f, g : X \to \mathbb{R}$ is defined by $f \otimes g(x, x') := f(x) g(x')$, $x, x' \in X$. The $d$-fold tensor product can be defined analogously.

\subsection{Anisotropic Gaussian kernels and their RKHSs} \label{sec::AniGauRKHS}

The anisotropic kernels can be defined as follows: 
\begin{definition}[Anisotropic kernel]
	A function $k_{\boldsymbol{w}} : X \times X \to \mathbb{R}$ is called an anisotropic kernel on $X$ with the shape parameter $\boldsymbol{w} = (w_1, \ldots, w_d)^\top$ if there exists a Hilbert space $H$ and a map $\Phi : X \to H$ such that for all $\boldsymbol{x}, \boldsymbol{x}' \in X$ we have
	\begin{align*}
	k_{\boldsymbol{w}}(\boldsymbol{x}, \boldsymbol{x}')
	= \langle \Phi (\boldsymbol{w}^T\boldsymbol{x}'), \Phi (\boldsymbol{w}^T\boldsymbol{x}) \rangle
	= \langle \Phi_{\boldsymbol{w}}(\boldsymbol{x}'), \Phi_{\boldsymbol{w}}(\boldsymbol{x}) \rangle,
	\end{align*}
    where $\boldsymbol{x} = (x_1, \ldots, x_d)^T$ and $\boldsymbol{x}' = (x'_1, \ldots, x'_d)^T$. $\Phi_{\boldsymbol{w}}$ is called a feature map and $H$ is called a feature space of $k_{\boldsymbol{w}}$.
\end{definition}

Careful observation of the definition of isotropic kernels, see e.g.~Definition 4.1 in \cite{StCh08}, will find that they can be taken as anisotropic kernels with the shape parameter $\boldsymbol{w}$ being an all-one vector. One commonly utilized anisotropic kernel is the \textit{anisotropic Gaussian kernel}. With the shape parameter $\boldsymbol{w} = (\gamma_1^{-1}, \ldots, \gamma_d^{-1})$, it takes the form:
\begin{align} \label{Aniso_Gaussian_Kernel}
k_{\boldsymbol{\gamma}} (\boldsymbol{x},\boldsymbol{x}')
:= k_{\boldsymbol{w}} (\boldsymbol{x},\boldsymbol{x}')
= \exp \bigl( - \|{\boldsymbol{w}}^T\boldsymbol{x}' - {\boldsymbol{w}}^T\boldsymbol{x}\|_2^2 \bigr)
= \exp \biggl( - \sum_{j=1}^d \frac{|x_j-x_j'|^2}{\gamma_j^2} \biggr),
\end{align}
where $\boldsymbol{\gamma} = (\gamma_1, \ldots, \gamma_d) \in (0,1]^d$ is called the multi-bandwidth of the anisotropic Gaussian kernels $k_{\boldsymbol{\gamma}}$.

Next, we are encouraged to determine an explicit formula for the RKHSs of anisotropic Gaussian RBF kernels. To this end, let us fix $\boldsymbol{\gamma}=(\gamma_1, \ldots, \gamma_d)$ where $\gamma_i > 0, i=1,\ldots, d$ and $d \in \mathbb{N}$.
For a given function $f : \mathbb{R}^d \to \mathbb{R}$ we define
\begin{align} \label{RKHS_Norm} 
\|f\|_{H_{\boldsymbol{\gamma}}} 
:= \biggl( \frac{2}{\pi} \biggr)^{d/2} 
\biggl( \prod_{i=1}^d \gamma_i \biggr)^{-1} 
\biggl( \int_{\mathbb{R}^d} |f(\boldsymbol{x})|^2 
\exp \biggl( - \sum_{i=1}^d \frac{4 x_i^2}{\gamma_i^2} \biggr)
\, d\boldsymbol{x} \biggr)^{1/2} \ ,
\end{align}
where $d\boldsymbol{x}$ stands for the Lebesgue measure on $\mathbb{R}^d$.
Furthermore, we write
\begin{align*}
H_{\boldsymbol{\gamma}} 
:= \bigl\{ f : \mathbb{R}^d \to \mathbb{R} \, | \, \|f\|_{\boldsymbol{\gamma}} < \infty \bigr\} \, .
\end{align*}
Obviously, $H_{\boldsymbol{\gamma}}$ is a function space with Hilbert norm $\|\cdot\|_{\boldsymbol{\gamma}}$. The following theorem shows that $H_{\boldsymbol{\gamma}}$ is the RKHS of the anisotropic Gaussian RBF kernel $k_{\boldsymbol{\gamma}}$.
\begin{theorem}[RKHS of the anisotropic Gaussian RBF]
	Let $\boldsymbol{\gamma}=(\gamma_1, \ldots, \gamma_d)$ where $\gamma_i > 0, i=1,\ldots, d$ and $d \in \mathbb{N}$. Then $(H_{\boldsymbol{\gamma}},\ \|\cdot\|_{H_{\boldsymbol{\gamma}}})$ is an RKHS and $k_{\boldsymbol{\gamma}}$ is its reproducing kernel. Furthermore, for $n \in \mathbb{N}_0:=\mathbb{N}\cup\{0\}$, let $e_{i,n} : \mathbb{R} \to \mathbb{R}$ be defined by
	\begin{align*}
	e_{i,n} (x_i) := \sqrt{\frac{2^n}{\gamma_i^{2n} n!}} \, x_i^n e^{- \gamma_i^{-2} x_i^2},
	\,\,\,\,\,\,\,\,\,\,
	\text{ for all } x_i \in \mathbb{R}.
	\end{align*}
Then the system $(e_{1,n_1} \otimes \cdots \otimes e_{d,n_d})_{n_1,\dots,n_d\geq 0}$ of functions $e_{1,n_1} \otimes \cdots \otimes e_{d,n_d}: \mathbb{R}^d \to \mathbb{R}$ defined by
\begin{align*}
	e_{1,n_1} \otimes \cdots \otimes e_{d,n_d}(z_1, \ldots, z_d) := \prod_{j=1}^d e_{j,n_j} (z_j),
	\,\,\,\,\,\,\,\,\,\, (z_1, \ldots, z_d) \in \mathbb{R}^d,
	\end{align*}
is an orthonormal basis of $H_{\boldsymbol{\gamma}}$.
\end{theorem}
The above theorem of orthonormal basis (OBS) of $H_{\boldsymbol{\gamma}}$ is in the same way as Theorem 4.38 in \cite{StCh08}. 
Therefore, we omit the proof here.
Note that the reproducing kernel of a RKHS is determined by an arbitrary ONB of this RKHS. Therefore, $k_{\boldsymbol{\gamma}}$ is its reproducing kernel of the RKHS $H_{\boldsymbol{\gamma}}$ which turns out to be the product function space of the RKHSs $H_{\gamma_i}$, that is,
$H_{\boldsymbol{\gamma}} = \otimes_{i=1}^d H_{\gamma_i}$,
where $H_{\gamma_i}$ is the RKHS of the one-dimensional Gaussian kernel
\begin{align} \label{1D_Gaussian_Kernel}
k_{\gamma_i} (x, x')
= \exp \biggl( - \frac{|x-x'|^2}{\gamma_i^2} \biggr),
\,\,\,\,\,\,\,\,\,\,\,\,\,\,
x, x' \in \mathbb{R}.
\end{align}

\subsection{Anisotropic Besov spaces} \label{sec::AniBesovSpace} 
Let us begin by introducing some function spaces we need.
Sobolev spaces \cite{AdFo03, Tr83} are one type of subspaces of $L_p(\nu)$. Let $\partial^{(\alpha)}$ be the $\alpha$-th weak derivative for a multi-index $\alpha = (\alpha_1, \ldots, \alpha_d) \in \mathbb{N}_0^d$ with $|\alpha| = \sum_{i=1}^d \alpha_i$. Then, for an integer $m \geq 0$, $1 \leq p \leq \infty$, and a measure $\nu$, the Sobolev space of order $m$ with respect to $\nu$ is defined by 
\begin{align*}
	W_p^m(\nu) := \{ f \in L_p(\nu): \partial^{(\alpha)} f \in L_p(\nu) \text{ exists for all } \alpha \in \mathbb{N}_0^d \text{ with } |\alpha| \leq m\}.
	\end{align*}
It is the space of all functions in $L_p(\nu)$ whose weak derivative up to order $m$ exist and are contained in $L_p(\nu)$. The Sobolev norm \cite{AdFo03} of the Sobolev space is given by
\begin{align*} 
	\|f\|_{W_p^m (\nu)} := \bigg( \sum_{|\alpha| \leq m} \big\| \partial^{(\alpha)} f \big\|_{L_p(\nu)}^p \bigg)^{\frac{1}{p}}.
	\end{align*}
In addition, we write $W_p^{0}(\nu) := L_p(\nu)$, and define $W_p^{m}(X) := W_p^{m}(\mu)$ for the Lebesgue measure $\mu$ on $X \subset \mathbb{R}^d$.

Another typical subspaces of $L_p(\nu)$ with a fine scale of smoothness which is commonly considered in the approximation theory, namely anisotropic Besov spaces. In order to clearly describe these function spaces, we need to introduce some device to measure the smoothness of function, which is the modulus of smoothness.

Let $X = \prod_{i=1}^d X^{(i)} \subset \mathbb{R}^d$ be a subset with non-empty interior, 
$\nu = \otimes_{i=1}^d \nu_i$ be an arbitrary measure on $X$ with marginal measure $\nu_i$ on $X^{(i)}$, $1 \leq i \leq d$. 
For a function $f : X \to \mathbb{R}$ with $f \in L_p(\nu)$ for some $p \in (0, \infty)$, $\boldsymbol{x} = (x_1, \ldots, x_d) \in X$, and $1 \leq i \leq d$, let $f_i := f_i(\cdot|\boldsymbol{x}) :\mathbb{R} \to \mathbb{R}$, $i =1, \ldots, d$ denote the univariate function 
\begin{align} \label{UnivariateFunction}
f_i(y) := f_i(\cdot|\boldsymbol{x}) := f(x_1, \ldots, x_{i-1}, y ,x_{i+1}, \ldots, x_d).
\end{align}
Now, we give the formal definition of the modulus of smoothness.

\begin{definition}[Modulus of smoothness]
	Let $X \subset \mathbb{R}^d$ be a subset with non-empty interior, $\nu$ be an arbitrary measure on $X$, and $f: X \to \mathbb{R}$ be a function with $f \in L_p(\nu)$ for some $p \in (0, \infty]$. For $\boldsymbol{r} = (r_1, \ldots, r_d) \in \mathbb{N}^d$, the $r_i$-th modulus of smoothness of $f$ in the direction of the variable $x_i$ is defined by
	\begin{align} \label{ModulusOfSmoothness}
	\omega_{r_i,{L_p(\nu_i)}} (f_i, t_i) 
	= \sup_{\boldsymbol{x} \in \mathbb{R}^d} 
	\sup_{0 < |h_i| \leq t_i} 
	\|\triangle_{h_i}^{r_i} (f_i, \,\cdot \, )\|_{L_p(\nu_i)} \ ,
	\end{align}
	where the $r_i$-th difference of $f$ in the direction of the variable $x_i$, $1 \leq i \leq d$, denoted as $\triangle_{h_i}^{r_i} (f_i,\, \cdot\, )$, is defined by
	\begin{align*}
		\triangle_{h_i}^{r_i} (f_i, x_i) = 
		\begin{cases}
			\sum_{j_i=0}^{r_i} {r_i \choose j_i} (-1)^{r_i-j_i} f_i(x_i + j_i h_i) 
			& \text{ if } x_i \in X_{r_i,h_i} 
			\\ 
			0 & \text{ if } x_i \notin X_{r_i,h_i}
		\end{cases}
	\end{align*}
for $\boldsymbol{j} = (j_1, \ldots, j_d) \in \mathbb{N}^d$, $\boldsymbol{h} = (h_1, \ldots, h_d) \in [0, \infty)^d$ and $X_{r_i,h_i} := \{ x_i \in X^{(i)} : x_i + s_i h_i \in X^{(i)} \text{ f.a. } s_i \in [0,r_i] \}$.
\end{definition}

To elucidate the idea of the modulus of smoothness, let us consider the case where $r_i=1$, $i=1, \ldots, d$. Then, we obtain
\begin{align*}
h_i^{-1}\triangle_{h_i}^{1} (f_i,\, x_i ) = \frac{f_i(x_i + h_i) - f_i (x_i)}{h_i} \xrightarrow{h_i \to 0} f_i' (x_i),
\end{align*}
if the derivative $f_i'$ of $f_i$ exists in $x_i$. As a result, $h_i^{-1}\triangle_{h_i}^{1} (f_i,\, x_i )$ equals the secant's slope and is bounded, if $f_i$ is differentiable at $x_i$. Analogously, $h_i^{-1}\triangle_{h_i}^{r_i} (f_i,\, x_i )$ is bounded, if, e.g.~second order derivatives exist.

It follows \cite{JoSc76} immediately that for all $f_i \in L_p(\mathbb{R}^d)$,  all $\boldsymbol{t} = (t_1, \ldots, t_d) \in \mathbb{R}_+^d$, and all $\boldsymbol{s} = (s_1, \ldots, s_d) \in \mathbb{R}_+^d$,
the modulus of smoothness with respect to $L_p(\nu_i)$ in the direction of the variable $x_i$ satisfies
\begin{align}\label{EigMod_1} 
\omega_{r_i,L_p(\mathbb{R}^d)} (f_i,t_i) 
\leq \biggl( 1 + \frac{t_i}{s_i} \biggr)^{r_i}
\omega_{r_i,L_p(\mathbb{R}^d)} (f_i,s_i) \, .
\end{align}

The modulus of smoothness \eqref{ModulusOfSmoothness}  in the direction of the variable $x_i$ can be used to define the scale of Besov spaces
in the direction of the variable $x_i$. 
For $\boldsymbol{\alpha} = (\alpha_1, \ldots, \alpha_d) \in \mathbb{R}^d$,
$1 \leq p, q \leq \infty$,  $r_i := \lfloor \alpha_i \rfloor +1$, $1 \leq i \leq d$, and an arbitrary measure $\nu = \otimes_{i=1}^d \nu_i$,
the Besov space $B_{p,q,i}^{\alpha_i}(\nu_i)$ 
in the direction of the variable $x_i$ is defined by
\begin{align*}
B_{p,q,i}^{\alpha_i}(\nu) 
:= \Bigl\{ f \in L_p(\nu) : |f|_{B_{p,q,i}^{\alpha_i}(\nu_i)} < \infty \Bigr\} \ ,
\end{align*}
where the seminorm $|\cdot\,|_{B_{p,q,i}^{\alpha_i}(\nu_i)}$ is defined by
\begin{align*}
|f|_{B_{p,q,i}^{\alpha_i}(\nu_i)} := 
\begin{cases}
\bigl( \int_0^\infty \bigl( t_i^{-\alpha_i} \omega_{r_i,L_p(\nu_i)} (f_i, t_i) \bigr)^q \frac{dt_i}{t_i} \bigr)^{1/q} 
& \text{ for } 1 \leq q < \infty, \\
\sup_{t_i>0} \bigl( t_i^{-\alpha_i} \omega_{r_i,L_p(\nu_i)} (f_i, t_i) \bigr) 
& \text{ for } q = \infty.
\end{cases}
\end{align*}
In both cases, the norm of $B_{p,q,i}^{\alpha_i}(\nu_i)$ can be defined by
\begin{align*}
\|f\|_{B_{p,q,i}^{\alpha_i}(\nu_i)} 
:= \|f\|_{L_p(\nu)} + |f|_{B_{p,q,i}^{\alpha_i}(\nu_i)} \ .
\end{align*}
In addition, for $q = \infty$, we often write 
$$
\mathrm{Lip}^*(\alpha_i, L_p(\nu_i)) := B_{p,\infty,i}^{\alpha_i}(\nu_i)
$$
and call $\mathrm{Lip}^*(\alpha_i,L_p(\nu_i))$ 
the generalized Lipschitz space of order $\alpha_i$.
Finally, if $\nu$ is the Lebesgue measure on $X$, we write
$B_{p,q,i}^{\alpha_i}(X_i) := B_{p,q,i}^{\alpha_i}(\nu_i)$.

\begin{definition}[Anisotropic Besov space]
	For 
	$\boldsymbol{p} = (p_1, \ldots, p_d) \in [1,\infty)^d$,
	$\boldsymbol{q} = (q_1, \ldots, q_d) \in [1,\infty]^d$, and
	$\boldsymbol{\alpha} = (\alpha_1, \ldots, \alpha_d) \in \mathbb{R}^d$,
	the anisotropic Besov space $B_{\boldsymbol{p},\boldsymbol{q}}^{\boldsymbol{\alpha}}(\nu)$
	is defined by
	\begin{align*}
	B_{\boldsymbol{p},\boldsymbol{q}}^{\boldsymbol{\alpha}}(\nu)
	:= \biggl\{ f \in L_p(\nu) : \sum_{i=1}^d \|f\|_{B_{p,q,i}^{\alpha_i}(\nu_i)} < \infty \biggr\} \ .
	\end{align*}
\end{definition}

In the case $p_i = p$ and $q_i = q$, $i = 1, \ldots, d$, we use the notation
\begin{align*}
B_{p,q}^{\boldsymbol{\alpha}}(\nu) :=
B_{\boldsymbol{p},\boldsymbol{q}}^{\boldsymbol{\alpha}}(\nu).
\end{align*}

\section{Main results} \label{sec::convergence_rates}
In this section, we present our main results: optimal learning rates for LS-SVMs using anisotropic Gaussian kernels for the non-parametric regression problem
based on the least squares loss $L : Y \times \mathbb{R} \to [0, \infty)$ defined by
$L(y, t) = (y - t)^2$.

\subsection{Convergence rates}

It is well known that, for the least squares loss, the function
$f_{L,\mathrm{P}}^* : \mathbb{R}^d \to \mathbb{R}$ defined by
$f_{L,\mathrm{P}}^*(x) = \mathbb{E}_\mathrm{P} (Y | x)$, $x \in \mathbb{R}^d$, 
is the only function for which the Bayes risk is
attained. Furthermore, some simple and well-known transformations show
\begin{align}\label{ExcessRiskLS}
\mathcal{R}_{L,\mathrm{P}}(f) - \mathcal{R}_{L,\mathrm{P}}^* 
= \|f - f_{L,\mathrm{P}}^*\|_{L_{2}(\mathrm{P}_X)}^2 \, . 
\end{align}
Note that, for all $t \in \mathbb{R}$ and $y \in [-M, M]$, the least squares loss 
can be clipped at $M > 0$ in the sense of Definition 2.22 in \cite{StCh08}.
To be precise, we denote the clipped value of some $t\ \in \mathbb{R}$ by $\wideparen{t}$, that is 
$\wideparen{t} := -M$ if $t < -M$, 
$\wideparen{t} := t$ if $t \in [-M, M]$, and
$\wideparen{t} := M$ if $t > M$.
It can be easily verified that the risks of 
the least squares loss satisfies $L(y, \wideparen{t}\,) \leq L(y,t)$, and 
therefore $\mathcal{R}_{L,\mathrm{P}}(\wideparen{f}) \leq \mathcal{R}_{L,\mathrm{P}}(f)$
holds for all $f : X \to \mathbb{R}$.

\begin{theorem} \label{LearningRates}
	Let $Y := [- M, M]$ for $M > 0$, and $\mathrm{P}$ be a distribution on $\mathbb{R}^d \times Y$ such that $X := \mathrm{supp} \mathrm{P}_X \subset B_{\ell_2^d}$ is a bounded domain with $\mu(\partial X) = 0$. 
	Furthermore, assume that $\mathrm{P}_X$ is a
	distribution on $\mathbb{R}^d$ that for all $1 \leq i \leq d$,
	the marginal distributions $\mathrm{P}_{X_i}$
	has a Lebesgue density $g_i \in L_q(\mathbb{R})$ for some $q \in [1, \infty]$.
	Moreover, let $f_{L,\mathrm{P}}^* : \mathbb{R}^d \to \mathbb{R}$ be a Bayes decision function such that
	$f_{L,\mathrm{P}}^* \in L_2(\mathbb{R}^d) \cap L_{\infty}(\mathbb{R}^d)$ as well as
	$f_{L,\mathrm{P}}^* \in B_{2s,\infty}^{\boldsymbol{\alpha}}(\mathbb{R}^d)$ for $\boldsymbol{\alpha} = (\alpha_1, \ldots, \alpha_d) \geq 1$ and $s \geq 1$ with $\frac{1}{q} + \frac{1}{s} = 1$. Let
	\begin{align} \label{MeanSmoothness}
	\alpha_0
	:= \biggl( \frac{1}{d} \sum_{j=1}^d \frac{1}{\alpha_i} \biggr)^{-1}.
	\end{align}
	Then, for all $\tau \geq 1$ and all $n \geq 1$,
	the SVM using the anisotropic RKHS $H_{\boldsymbol{\gamma}}$ and the least squares loss $L$ with 
	\begin{align*}
	\lambda_n = c_1 n^{-1},
	\qquad 
	\text{ and }
	\qquad 
	\gamma_{i,n} = c_{2,i} n^{- \frac{\alpha_0}{\alpha_i(2 \alpha_0 + d)}}, \, 1 \leq i \leq d, 
	\end{align*}
	we have 
	\begin{align}\label{RatesLSSVM}
	\mathcal{R}_{L,\mathrm{P}}(\wideparen{f}_{\mathrm{D},\lambda_n,\gamma_n}) - \mathcal{R}_{L,\mathrm{P}}^*
	\leq C (\log n)^{d+1} n^{-\frac{2\alpha_0}{2\alpha_0+d}}
	\end{align}
	with probability $\mathrm{P}^n$ not less than $1- e^{-\tau}$. 
	Here, $c_1 > 0$ and $c_{2,i} > 0$, $1 \leq i \leq d$, are user-specified constants and $C > 0$ is a constant independent of $n$.
\end{theorem}

Note that for isotropic cases, the overall smoothness is depend on the worst smoothness of all dimensions. In other words, one poor smoothness along certain dimension may lead to unsatisfying convergence rates of the decision functions, even when smoothness of other dimensions is well-behaved. In contrast, the anisotropic cases are more appropriate for one poor smoothness of certain dimension will not jeopardize the overall good smoothness much by 
embodying smoothness by $\alpha_0$ in \eqref{MeanSmoothness}. This $\alpha_0$ is called the \textit{effective smoothness} \cite{HoLe02a}, or the \textit{exponent of global smoothness} \cite{Birge86a}. Moreover, we can still precisely characterize the \emph{anisotropy} by considering the dimension-specific smoothness vector $\boldsymbol{a} = (a_1, \ldots, a_d)$ with $a_i = \alpha_0/\alpha_i$, $i = 1, \ldots, d$.

In the statistical literature, optimal rates of convergence in anisotropic H\"{o}lder, Sobolev and Besov spaces have been studied in \cite{IbHa81,Nussbaum85a,HoLe02a}. The theoretical optimal learning rate for a function with smoothness $\alpha_i$ along the $i$-th dimension is given by  $n^{-2\alpha_0/(2\alpha_0+d)}$, see e.g. \cite{HoLe02a}. Therefore, our established convergence rates in \eqref{RatesLSSVM} match the theoretical optimal ones up to the logarithmic factor $(\log n)^{d+1}$.
Other published convergence rates for anisotropic cases include ones learned by Gaussian process \cite{IbHa81}. With optimal rates learned by SVMs based on anisotropic Gaussian kernel, the results we obtained is in line with these existing ones derived via different algorithms.
Moreover, when considering our rates in the isotropic classes where $\alpha_i = \alpha$ for all $i = 1, \ldots, d$, our rates become $\mathcal{O} ((\log n)^{d+1} n^{-2\alpha/(2\alpha+d)})$ and it is better than the learning rates $\mathcal{O}(n^{-2\alpha/(2\alpha+d) + \xi})$ via SVMs based on isotropic kernel obtained in \cite{eberts2013optimal}. Furthermore, the well-known theoretical optimal rate for isotropic cases is $n^{-2\alpha/(2\alpha+d)}$, see \cite{Stone82}, and our learning rates coincide with it up to the logarithmic factor $(\log n)^{d+1}$.

Though literature often focuses on the isotropic class, this assumption of isotropy would lead to loss of efficiency if the regression function actually belongs to an anisotropic class. Moreover, this inefficiency is getting worser when the dimension becomes higher. Therefore, assumption of anisotropy which treats isotropy as a special case may be a better choice. In addition, the assumption of anisotropy also shows its advantages in facing sparse regression functions, i.e., if the regression function depends only on a small subset of coordinates $I \subset \{ 1, \ldots, d \}$. Therefore, the effective smoothness in \eqref{MeanSmoothness} will depend less on smoothness along some certain dimensions, and thus become larger. In this manner, the learning rates in \eqref{RatesLSSVM} can be further significantly improved.

\begin{theorem} \label{LearningRatesSubset}
	Let the assumptions on the distribution $\mathrm{P}$ 
	and the Bayes decision function
	$f_{L,\mathrm{P}}^* : \mathbb{R}^d \to \mathbb{R}$
	in Theorem \ref{LearningRates}
	hold. Moreover,
	suppose $f_{L,\mathrm{P}}^*$ belongs to 
	$B_{2s,\infty}^{\boldsymbol{\alpha}}(\mathbb{R}^I)$
	for some subset $I$ of $\{ 1, \ldots, d \}$. Let 
	\begin{align} \label{MeanSmoothnessSubset}
	\alpha_0^I 
	:= \biggl( \frac{1}{d} \sum_{i \in I} \frac{1}{\alpha_i} \biggr)^{-1}.
	\end{align}
	Then, for all $\tau \geq 1$ and all $n \geq 1$,
	the SVM using the anisotropic RKHS $H_{\boldsymbol{\gamma}}$ and the least squares loss $L$ with 
	\begin{align*}
	\lambda_n = c_1 n^{-1}, 
	\qquad 
	\text{ and }
	\qquad 
	\gamma_{i,n} = c_{2,i} n^{- \frac{\alpha_0^I}{\alpha_i(2 \alpha_0^I + d)}}, \, i \in I, 
	\end{align*}
	we have 
	\begin{align}\label{RatesLSSVMSubset}
	\mathcal{R}_{L,\mathrm{P}}(\wideparen{f}_{\mathrm{D},\lambda_n,\gamma_n}) - \mathcal{R}_{L,\mathrm{P}}^*
	\leq C (\log n)^{|I|+1} n^{-\frac{2\alpha_0^I}{2\alpha_0^I+d}}
	\end{align}
	with probability $\mathrm{P}^n$ not less than $1- e^{-\tau}$. 
	Here, $c_1 > 0$ and $c_{2,i} > 0$, $i \in I$, are user-specified constants and $C > 0$ is a constant independent of $n$.
\end{theorem}

The proof of Theorem \ref{LearningRatesSubset} will be omitted as it is similar to the previous theorem. We only mention that the exponents of the logarithmic terms depend on the capacity of the underlying RKHSs.

\subsection{Rate Analysis} \label{sec::rate_analysis}
In this section, we compare our results with previously obtained learning rates for SVMs for regression.
To this end, according to Theorem 9 in \cite{steinwart2009optimal}, we need to verify two conditions, which are $\mu_i(T_k) \sim i^{-\frac{1}{p}}$ and $\| T_k^{\beta/2}  f\|_{\infty} \leq c \| f \|_{L_2(\nu)}$, $f \in L_2(\nu)$, where $T_k$ is the integral operator (see, (5) in \cite{steinwart2009optimal}).

First of all, we prove that $\mu_i(T_k) \sim i^{-\frac{1}{p}}$ holds. With the help of Theorem 15 in \cite{steinwart2009optimal}, $\mu_i(T_k) \leq a i^{-\frac{1}{p}}$ is equivalent to 
\begin{align} \label{inq_1}
	e_i (\mathrm{id}: H \to L_2(P_X)) \leq \sqrt{a} i^{-\frac{1}{2p}}
	\end{align}
modulo a constant only depending on $p$. If we denote $\ell_{\infty}$ the space of all bounded functions on $X$, then the above inequality \eqref{inq_1} will be satisfied if the following more classical, distribution-free entropy number assumption
\begin{align} \label{inq_2}
	e_i (\mathrm{id}: H \to \ell_{\infty} (X)) \leq \sqrt{a} i^{-\frac{1}{2p}}
	\end{align}
is satisfied. Since the establishment of \eqref{inq_2} can be verified by Proposition \ref{EntropyNumberRKHS} that will be mentioned later, we manage to prove the establishment of $\mu_i(T_k) \sim i^{-\frac{1}{p}}$.

Now, we verify that the second condition $\| T_k^{\beta/2} f \|_{\infty} \leq c \| f \|_{L_2(\nu)}$, $f \in L_2(\nu)$ also holds. Since the proof of Theorem 4.1 in \cite{cucker2007learning} shows that the image of $T_k^{\beta/2}$ is continuously embedded into the real interpolation space $[L_2(\nu),\, H]_{\, \beta, \infty}$, therefore the second condition is satisfied if we can prove that $[L_2(\nu),\, H]_{\, \beta, \infty}$ is continuously embedded in $\ell_{\infty}(X)$. 

In order to present a more concrete example,
we need to introduce some notations. 
In the following, for $\boldsymbol{v} = (v_1, \cdots, v_d) \in \mathbb{N}^d$,
we denote $\overline{\boldsymbol{v}} := \max \{ v_i, \  i = 1, \ldots, d \}$ and
$\underline{\boldsymbol{v}} := \min \{ v_i, \  i = 1, \cdots, d \}$.
Let us now consider the case where $H = W^{\boldsymbol{m}}$ for some $\underline{\boldsymbol{m}} > d/2$. 
If $\mathrm{P}_X = \nu$ has a Lebesgue density that is bounded away from $0$ and $\infty$, then we have
\begin{align*}
B_{2,\, \infty}^{\boldsymbol{\alpha}} (X) = [L_2(\mathrm{P}_X),\, W^{\boldsymbol{m}}(X)]_{\, \beta, \infty},
\end{align*}
where $\boldsymbol{\alpha}:= \beta \boldsymbol{m}$.
Consequently, by Corollary 6 in \cite{steinwart2009optimal}, we can obtain the learning rates $n^{- 2\underline{\boldsymbol{\alpha}} /( 2\underline{\boldsymbol{\alpha}} + d  )}$ whenever $f_p^* \in B_{2,\, \infty}^{\,{\boldsymbol{\alpha}} }(X)$. Conversely, according to the Imbedding Theorem for Anisotropic Besov Space, see Theorem \ref{ITABS} in the Appendix, $B_{2,\, \infty}^{\, {\boldsymbol{\alpha}}} (X) $ can be continuously embedded into $\ell_{\infty}(X)$ for all 
$\underline{\boldsymbol{\alpha}} > d/2$, and therefore, Theorem 9 in \cite{steinwart2009optimal} shows that the learning rates $n^{- 2\underline{\boldsymbol{\alpha}} /( 2\underline{\boldsymbol{\alpha}} + d  )}$ is asymptotic optimal for such ${\boldsymbol{\alpha}}$. However, if $\underline{\boldsymbol{\alpha}} \in (0, d/2]$, we can still obtain the rates $n^{- 2\underline{\boldsymbol{\alpha}} /( 2\underline{\boldsymbol{\alpha}} + d  )}$, but we no longer know whether they are optimal in the minimax sense.

It is noteworthy that, when the target functions reside in the anisotropic Besov space, as shown in Theorem \ref{LearningRates},
we obtain 
the optimal learning rates $n^{-2\alpha_0/(2 \alpha_0 + d)}$
up to certain logarithmic factor.
There, $\alpha_0 := (({\Sigma_{j=1}^d 1/ \alpha_i})/d)^{-1}$, and the decision functions reside in the RKHSs induced by the anisotropic Gaussian RBF kernel. While, when using an anisotropic Sobolev space as the underlying RKHS, the learning rates obtained are $n^{- 2\underline{\boldsymbol{\alpha}} /( 2\underline{\boldsymbol{\alpha}} + d  )}$. It can be apparently observed that since $\alpha_0 > \underline{\boldsymbol{\alpha}} $, our learning rates in Theorem \ref{LearningRates} are faster than that with the anisotropic Sobolev space.

\section{Error Analysis} \label{sec::error_analysis}

\subsection{Bounding the Sample Error Term}

Aiming at proving the new oracle inequality in the Proposition \ref{pro::sample}, there is a need for us to control the capacity of the RKHS $H_{\boldsymbol{\gamma}}$ where we use the entropy numbers, see e.g. \cite{CaSt90} or Definition A.5.26 in \cite{StCh08}. The definition of the entropy numbers is presented as follows:

\begin{definition}
	Let $S : E \to F$ be a bounded, linear operator between the
	normed spaces $E$ and $F$ and $i \geq 1$ be an integer. 
	Then the $i$-th (dyadic) entropy number of $S$ is defined by
	\begin{align*}
	e_i(S) := \inf
	\biggl\{
	\varepsilon > 0 : \exists t_1, \ldots, t_{2^{i−1}} \in S B_E 
	\text{ such that } S B_E \subset \bigcup_{j=1}^{2^{i-1}} 
	(t_j + \varepsilon B_F)
	\biggr\}
	\end{align*}
	where the convention $\inf \emptyset := \infty$ is used.
\end{definition}

The following proposition with regard to the capacity of $H_{\boldsymbol{\gamma}}$ can be derived by Theorem 7.34 and Corollary 7.31 in \cite{StCh08}.

\begin{proposition} \label{EntropyNumberRKHS}
	Let $X \subset \mathbb{R}^d$ be a closed Euclidean ball. 
	Then there exists a constant $K > 0$, such
	that, for all $p \in (0, 1)$, $\boldsymbol{\gamma} \in (0, 1]^d$ and $i \geq 1$, we have
	\begin{align*}
	e_i ( \mathrm{id} : H_{\boldsymbol{\gamma}} \to \ell_{\infty}(X) )
	\leq (3 K)^{\frac{1}{p}}
	\biggl( \frac{d+1}{e p} \biggr)^{\frac{d+1}{p}}
	\biggl( \prod_{i=1}^d \gamma_i \biggr)^{-\frac{1}{p}}
	i^{-\frac{1}{p}}.
	\end{align*}
\end{proposition}

Having developed the above proposition, we are now able to derive the oracle inequality for the least squares loss as follows:

\begin{proposition} \label{pro::sample}
	Let $X \subset B_{\ell_2^d}$, $Y := [-M, M] \subset \mathbb{R}$ be a closed subset with $M > 0$ and $\mathrm{P}$ be a distribution on $X \times Y$. Furthermore, let $L: Y \times \mathrm{R} \to [0, \infty)$ be the least squares loss, $k_{\boldsymbol{\gamma}}$ be the anisotropic Gaussian RBF kernel over $X$ with multi-bandwidth $\boldsymbol{\gamma} = (\gamma_1, \ldots, \gamma_d) \in (0, 1]^d$ and $H_{\boldsymbol{\gamma}}$ be the associated RKHS. Fix an $f_0 \in H_{\boldsymbol{\gamma}}$ and a constant $B_0 \geq 4 M^2$ such that $\|L \circ f_0\| \leq B_0$. Then, for all fixed $\varrho \geq 1$ and $\lambda > 0$ and $p \in (0, 1)$, the SVM using $H_{\boldsymbol{\gamma}}$ and $L$ satisfies
	\begin{align}
	\lambda \|f_{\mathrm{D},\lambda,\boldsymbol{\gamma}}\|_{H_{\boldsymbol{\gamma}}}^2 
	& + \mathcal{R}_{L,\mathrm{P}}(\wideparen{f}_{\mathrm{D},\lambda,\boldsymbol{\gamma}}) 
	- \mathcal{R}^{*}_{L,\mathrm{P}} 
	\nonumber\\
	& \leq 9 \bigl( \lambda \|f_0\|_{H_{\boldsymbol{\gamma}}}^2 
	+ \mathcal{R}_{L,\mathrm{P}}(f_0) 
	- \mathcal{R}_{L,\mathrm{P}}^* \bigr)
	+ K(p) \biggl( \frac{a^{2p}}{\lambda^p n} \biggr)
	+ \frac{(3456 M^2 + 15 B_0) (1 + \log 3) \varrho}{n}
	\label{OracleInequality}
	\end{align}
with probability $\mathrm{P}^n$ not less that $1 - e^{-\varrho}$, 
where $K(p)$ is a constant only depending on $p$ and $M$.
\end{proposition}

\subsection{Bounding the Approximation Error Term} \label{sec::ApproxError}

In this section, we consider bounding the approximation
error of some function contained in the RKHS $H_{\gamma}$, which is defined by
\begin{align}\label{ApprErrFunc}
A_{\boldsymbol{\gamma}}(\lambda) 
:= \inf_{f \in H_{\boldsymbol{\gamma}}} \lambda \|f\|^2_{H_{\boldsymbol{\gamma}}} + \mathcal{R}_{L,\mathrm{P}}(f) - \mathcal{R}^{*}_{L,\mathrm{P}} \ ,
\end{align}
where the infimum is actually attained by a unique element $f_{\mathrm{P},\lambda,\boldsymbol{\gamma}} \in H_{\boldsymbol{\gamma}}$,
see Lemma 5.1 and Theorem 5.2 in \cite{StCh08}. 
To this end, it suffices to find a function $f_0 \in H_{\boldsymbol{\gamma}}$
such that both the regularization term $\lambda \|f_0\|_{H_{\boldsymbol{\gamma}}}^2$ 
and the excess risk $\mathcal{R}_{L,\mathrm{P}}(f_0) - \mathcal{R}^{*}_{L,\mathrm{P}}$ are small. 
To construct this function $f_0$ we define, for 
$\boldsymbol{x} = (x_1, \ldots, x_d) \in \mathbb{R}^d$,
$\boldsymbol{\gamma} = (\gamma_1, \ldots, \gamma_d) \in \mathbb{R}_+^d$,
$\boldsymbol{r} = (r_1, \ldots, r_d) \in \mathbb{N}^d$,
and $\boldsymbol{j} = (j_1, \ldots, j_d) \in \mathbb{N}^d$,
the function $K : \mathbb{R}^d \to \mathbb{R}$ by
\begin{align}\label{tilde_K}
K(\boldsymbol{x}) := \prod_{i=1}^d K_i(x_i)
\end{align}
with functions $K_i : \mathbb{R} \to \mathbb{R}$ defined by
\begin{align}\label{tilde_K_i}
K_i(x_i) := \left( \frac{2}{\pi} \right)^{1/2} 
\sum_{j_i=1}^{r_i} {r_i \choose j_i} \frac{(-1)^{1-j_i}}{j_i \gamma_i} 
\exp \biggl( - \frac{2 |x_i|^2}{j_i^2 \gamma_i^2} \biggr),
\,\,\,\,\,\,\,\,
1 \leq i \leq d.
\end{align}
Assume that there exists a Bayes decision function
$f_{L,\mathrm{P}}^* \in L_2(\mathbb{R}^d) \cap L_{\infty}(\mathbb{R}^d)$, 
then we define $f_0$ by 
\begin{align}\label{f_0}
f_0(\boldsymbol{x}) 
:= K * f_{L,\mathrm{P}}^*(\boldsymbol{x})
= \int_{\mathbb{R}^d} K(\boldsymbol{x} - \boldsymbol{t}) 
f_{L,\mathrm{P}}^*(\boldsymbol{t}) \, d\boldsymbol{t} \ ,
\,\,\,\,\,\,\,\,
\boldsymbol{x} \in \mathbb{R}^d.
\end{align}

In order to show that $f_0$ is indeed a suitable function to bound \eqref{ApprErrFunc}, there is a need for us to first ensure that $f_0$ is contained in $H_{\boldsymbol{\gamma}}$. Moreover, we need to bound both of the excess risk of $f_0$ and the $H_{\boldsymbol{\gamma}}$-norm. Proposition \ref{ErrorEstimation} gives the bound of the excess risk with the help of the modulus of smoothness and Proposition \ref{ConvolutionInRKHS} focus on the estimation of the regularization term.

\begin{proposition}\label{ErrorEstimation}
	Let us fix some $q \in [1, \infty)$. Furthermore, assume that $\mathrm{P}_X$ is a
	distribution on $\mathbb{R}^d$ that for all $1 \leq i \leq d$,
	the marginal distributions $\mathrm{P}_{X_i}$
	has a Lebesgue density $g_i \in L_p(\mathbb{R})$ for some $p \in [1, \infty]$. 
	Let $f : \mathbb{R}^d \to \mathbb{R}$ be such that $f \in L_q(\mathbb{R}^d) \cap L_{\infty}(\mathbb{R}^d)$. 
	Then, for $\boldsymbol{r} =(r_1, \ldots, r_d) \in \mathbb{N}^d$, $\boldsymbol{\gamma} = (\gamma_1, \ldots, \gamma_d) \in \mathbb{R}_+^d$, and $s \geq 1$ with $1 = \frac{1}{s} + \frac{1}{p}$, we have
	\begin{align*}
	\|K * f - f\|_{L_q(\mathrm{P}_X)}^q
	\leq \sum_{i=1}^d c_{r_i,q} \|g_i\|_{L_p(\mathbb{R})} \, \omega_{r_i,L_{q s}(\mathbb{R})}^q (f_i, \gamma_i/2) \ ,
	\end{align*}
	where $c_{r_i,q}$ are constants only depending on $r_i$ and $q$.
\end{proposition}

Next, the following proposition aims at bounding the regularization term and proving that the convolution of a function from $L_2(\mathbb{R}^d)$ with $K$ is contained in the RKHS $H_{\boldsymbol{\gamma}}$. The following result is considered to provide a helpful supremum bound.

\begin{proposition} \label{ConvolutionInRKHS}
	Let $f \in L_2(\mathbb{R}^d)$, $H_{\boldsymbol{\gamma}}$ be the RKHS of the anisotropic Gaussian kernel $k_{\boldsymbol{\gamma}}$ over $X \subset \mathbb{R}^d$ with $\boldsymbol{\gamma} = (\gamma_1, \ldots, \gamma_d) \in \mathbb{R}_+^d$ 
	and $K : \mathbb{R}^d \to \mathbb{R}$ be defined by \eqref{tilde_K} for a fixed $\boldsymbol{r} = (r_1, \ldots, r_d) \in \mathbb{N}^d$. Then we have $K * f \in H_{\boldsymbol{\gamma}}$ with
	\begin{align*}
	\|K * f\|_{H_{\boldsymbol{\gamma}}} 
	\leq 
	\pi^{-\frac{d}{4}} \prod_{i=1}^d (2^{r_i} - 1) \|f\|_{L_2(\mathbb{R}^d)} 
	\prod_{i=1}^d \gamma_i^{-\frac{1}{2}} \ .
	\end{align*}
	Moreover, if $f \in L_{\infty}(\mathbb{R}^d)$, we have
	\begin{align*}
	|K * f(\boldsymbol{x})| \leq \prod_{i=1}^d (2^{r_i} - 1) \|f\|_{L_{\infty}(\mathbb{R}^d)} \ , 
	\,\,\,\,\,\,\,\,\,
	\boldsymbol{x} \in X \ .
	\end{align*}
\end{proposition}

The following theorem gives an upper bound for the approximation error function
$A_{\boldsymbol{\gamma}}(\lambda)$.

\begin{theorem}\label{ApproximationError}
	Let $L$ be the least squares loss, 
	$\mathrm{P}$ be the probability distribution on $\mathbb{R}^d \times Y$,
	and the marginal distributions $\mathrm{P}_{X_i}$
	has a Lebesgue density $g_i \in L_p(\mathbb{R})$ for some $p \in [1, \infty]$. 
	Let $f : \mathbb{R}^d \to \mathbb{R}$ be such that 
	$f \in L_q(\mathbb{R}^d) \cap L_{\infty}(\mathbb{R}^d)$.
	such that $X := \mathrm{supp} \mathrm{P}_X \subset B_{\ell_2^d}$ is a bounded domain with 
	$\mu(\partial X) = 0$. 
	Furthermore, assume that $\mathrm{P}_X$ is a
	distribution on $\mathbb{R}^d$ that for all $1 \leq i \leq d$,
	the marginal distributions $\mathrm{P}_{X_i}$
	has a Lebesgue density $g_i \in L_q(\mathbb{R})$ for some $q \in [1, \infty]$.
	Moreover, let $f_{L,\mathrm{P}}^* : \mathbb{R}^d \to \mathbb{R}$ be a Bayes decision function such that $f_{L,\mathrm{P}}^* \in L_2(\mathbb{R}^d) \cap L_{\infty}(\mathbb{R}^d)$ as well as
	$f_{L,\mathrm{P}}^* \in B_{2s,\infty}^{\boldsymbol{\alpha}}(\mathbb{R}^d)$ for $\boldsymbol{\alpha} = (\alpha_1, \ldots, \alpha_d) \geq 1$ and $s \geq 1$ with $\frac{1}{q} + \frac{1}{s} = 1$. 
	Then, for all $\boldsymbol{\gamma} \in (0, 1]^d$ and $\lambda > 0$, there holds
	\begin{align}\label{ApproximationErrorEstimation}
	\lambda \|f_0\|_{H_{\boldsymbol{\gamma}}}^2 + \mathcal{R}_{L,\mathrm{P}}(f_0) - \mathcal{R}_{L,\mathrm{P}}^* 
	\leq C_1 \lambda \prod_{i=1}^d \gamma_i^{-1} + 
	C_s \sum_{i=1}^d \gamma_i^{2 \alpha_i},
	\end{align}
	where the constant $C_1 > 0$ and $C_s > 0$ is a constant depending on the smoothness $s$.
\end{theorem}

\section{Proofs} \label{sec::proofs}

\begin{proof}[Proof of Proposition \ref{EntropyNumberRKHS}]
	By Lemma 4.2 in \cite{BhPaDu14},
	the covering numbers of unit ball $B_{\boldsymbol{\gamma}}$ of
	the Gaussian RKHS $H_{\boldsymbol{\gamma}}(X)$ for all $\boldsymbol{\gamma} \in (0, 1)^d$ and 
	$\varepsilon \in (0, \frac{1}{2})$ satisfy
	\begin{align} \label{CoveringNumberRKHS}
	\log \mathcal{N} (B_{\boldsymbol{\gamma}}, \|\cdot\|_{\infty}, \varepsilon)
	\leq K \biggl( \prod_{i=1}^d \gamma_i \biggr)^{-1}
	\biggl( \log \frac{1}{\varepsilon} \biggr)^{d+1}.
	\end{align}
	where $K > 0$ is a constant depending only on $d$. 
	From this, we conclude that
	\begin{align*} 
	\sup_{\varepsilon \in (0, \frac{1}{2})} 
	\varepsilon^p \log \mathcal{N} (B_{\boldsymbol{\gamma}}, \|\cdot\|_{\infty}, \varepsilon)
	\leq K \biggl( \prod_{i=1}^d \gamma_i \biggr)^{-1}
	\sup_{\varepsilon \in (0, \frac{1}{2})} 
	\varepsilon^p \biggl( \log \frac{1}{\varepsilon} \biggr)^{d+1}.
	\end{align*}
	Let $h(\varepsilon) := \varepsilon^p ( \log \frac{1}{\varepsilon} )^{d+1}$.
	In order to obtain the optimal value of $h(\varepsilon)$, we differentiate it with
	respect to $\varepsilon$
	\begin{align*} 
	\frac{d h(\varepsilon)}{d \varepsilon}
	= p \varepsilon^{p-1} \biggl( \log \frac{1}{\varepsilon} \biggr)^{d+1}
	- \varepsilon^p (d+1) \biggl( \log \frac{1}{\varepsilon} \biggr)^d \log \frac{1}{\varepsilon},
	\end{align*}
	and set $\frac{d h(\varepsilon)}{d \varepsilon} = 0$ which gives
	\begin{align*} 
	\log \frac{1}{\varepsilon} = \frac{d+1}{p}
	\Rightarrow
	\varepsilon^* = \frac{1}{e^{\frac{d+1}{p}}}. 
	\end{align*}
	By plugging $\varepsilon^*$ into $h(\varepsilon)$, we obtain
	\begin{align*} 
	h(\varepsilon^*) = \biggl( \frac{d+1}{e p} \biggr)^{d+1},
	\end{align*}
	and consequently, the covering numbers \eqref{CoveringNumberRKHS} are
	\begin{align*} 
	\log \mathcal{N} (B_{\boldsymbol{\gamma}}, \|\cdot\|_{\infty}, \varepsilon)
	\leq K \biggl( \frac{d+1}{e p} \biggr)^{d+1}
	\biggl( \prod_{i=1}^d \gamma_i \biggr)^{-1}
	\varepsilon^{-p},
	\end{align*}
	where 
	$$
	a := K \biggl( \frac{d+1}{e p} \biggr)^{d+1}
	\biggl( \prod_{i=1}^d \gamma_i \biggr)^{-1}.
	$$ 
	Now, by inverse implication of Lemma 6.21 in \cite{StCh08}, see also Exercise
	6.8 in \cite{StCh08}, the bound on entropy number of the anisotropic Gaussian RBF kernel is
	\begin{align*}
	e_i ( \mathrm{id} : H_{\boldsymbol{\gamma}} \to \ell_{\infty}(X) )
	\leq (3 a)^{\frac{1}{p}} i^{- \frac{1}{p}}
	= (3 K)^{\frac{1}{p}}
	\biggl( \frac{d+1}{e p} \biggr)^{\frac{d+1}{p}}
	\biggl( \prod_{i=1}^d \gamma_i \biggr)^{-\frac{1}{p}}
	i^{-\frac{1}{p}}.
	\end{align*}
\end{proof}

\begin{proof}[Proof of Proposition \ref{pro::sample}]
	First of all, note that, for all $t \in \mathbb{R}$ and $y \in [-M, M]$, the least squares loss satisfies $L(y, \wideparen{t}) \leq L(y, t)$, that is, it can be clipped at $M>0$ (see \cite{steinwart2009optimal}). 
	Note that the least squares loss is locally Lipschitz continuous in the sense of Definition 2.18 in \cite{StCh08} with the local Lipschitz constant $|L|_{M,1} = 4 M$. 
	Moreover, it is easy to verify that the least squares loss satisfies the supremum bound
	\begin{align*}
	L(y,t) = (y-t)^2 \leq 4 M^2
	\end{align*}
	and the variance bound
	\begin{align*}
	\mathbb{E}_{\mathrm{P}} \bigl( L \circ \wideparen{f} - L \circ f_{L,\mathrm{P}}^* \bigr)^2 
	\leq 16 M^2 \mathbb{E}_{\mathrm{P}} \bigl( L \circ \wideparen{f} - L \circ f_{L,\mathrm{P}}^* \bigr)
	\end{align*}
	for all $y \in Y$, $t \in [-M, M]$, see also Example 7.3 in \cite{StCh08}. According to the above Proposition \ref{EntropyNumberRKHS}, we also verify the condition of the entropy numbers. We denote here that $K(p)$ is defined by
	\begin{align}
	K(p) = \max \{ 43200 \cdot 2^{2p} M^2 C_1(p)^2,
	360 \cdot 480^p M^2 C_2(p)^{1+p}, 8 M^2 \}.
	\label{K_p}
	\end{align}
	Here, the constants $C_1(p)$ and $C_2(p)$ are derived 
	in the proof of Theorem 7.16 in \cite{StCh08}, that is
	\begin{align*}
	C_1(p) := \frac{2 \sqrt{\log 256} C_p^p}{(\sqrt{2} - 1) (1 - p) 2^{p/2}}
	\,\,\,\,\,\,
	\text{ and }
	\,\,\,\,\,\,
	C_2(p) := \biggl( \frac{8 \sqrt{\log 16} C_p^p}{(\sqrt{2} - 1) (1 - p) 4^p} \biggr)^{\frac{2}{1+p}},
	\end{align*}
	where 
	\begin{align*}
	C_p := \frac{\sqrt{2} - 1}{\sqrt{2} - 2^{\frac{2p-1}{2p}}} \cdot \frac{1-p}{p},
	\end{align*}
	see Lemma 7.15 in \cite{StCh08}.
\end{proof}

\begin{proof}[Proof of Proposition \ref{ErrorEstimation}]
	First of all, we show $f \in L_q(\mathrm{P}_X)$. 
	Because of the assumption $f \in L_q(\mathbb{R}^d) \cap L_{\infty}(\mathbb{R}^d)$, 
	we have $f \in L_u(\mathbb{R}^d)$ and thus $f \in L_u(X)$ 
	for every $u \in [q, \infty]$ and $q \in [1, \infty]$. In addition,
	\begin{align*}
	\|f\|_{L_q(\mathrm{P}_X)}
	\biggl( \int_{\mathbb{R}^d} |f(\boldsymbol{x})|^q 
	\, d\mathrm{P}_X(\boldsymbol{x}) \biggr)^{1/q}
	= \biggl( \int_X |f(\boldsymbol{x})|^q 
	\, d\mathrm{P}_X(\boldsymbol{x}) \biggr)^{1/q}
	\leq \|f\|_{L_{\infty}(X)} < \infty
	\end{align*}
	holds, i.e.~$f \in L_q(\mathrm{P}_X)$ for all $q \in [1, \infty)$. It remains to show
	\begin{align*}
	\|K * f - f\|_{L_q(\mathrm{P}_X)}^q
	\leq \sum_{i=1}^d c_{r_i,q} \|g_i\|_{L_p(\mathbb{R})} \, \omega_{r_i,L_{q s}(\mathbb{R})}^q (f, \gamma_i/2) \ .
	\end{align*}
	To this end, we use the translation invariance of the Lebesgue measure and 
	$K(\boldsymbol{x}) = K(-\boldsymbol{x})$ ($\boldsymbol{x} \in \mathbb{R}^d$) to obtain,
	for $\boldsymbol{r} = (r_1, \ldots, r_d) \in \mathbb{N}^d$,
	$\boldsymbol{\gamma} = (\gamma_1, \ldots, \gamma_d) \in \mathbb{R}_+^d$,
	and $\boldsymbol{x} = (x_1, \ldots, x_d) \in \mathbb{R}^d$,
	\begin{align*}
	K * f(\boldsymbol{x}) 
	& = \int_{\mathbb{R}^d} K(\boldsymbol{x} - \boldsymbol{t}) f(\boldsymbol{t}) \, d\boldsymbol{t} 
	\\
	& = \int_{\mathbb{R}} \cdots \int_{\mathbb{R}} 
	\prod_{i=1}^d K_i(x_i - t_i) f(t_1, \ldots, t_d) \, dt_d \cdots \, dt_1
	\\
	& = \int_{\mathbb{R}^{d-1}} \prod_{i=1}^{d-1}
	K_i(x_i - t_i)
	\biggl( \int_{\mathbb{R}} 
	K_d(x_d - t_d) f(\boldsymbol{t}) \, dt_d \biggr)
	\, dt_{d-1} \cdots \, dt_1.
	\end{align*}
	A simple calculation show that
	\begin{align*}
	\int_{\mathbb{R}} K_d(x_d - t_d) f(\boldsymbol{t}) \, dt_d
	& = \int_{\mathbb{R}} \sum_{j_d=1}^{r_d} {r_d \choose j_d} \frac{(-1)^{1-j_d}}{j_d \gamma_d} \left( \frac{2}{\pi} \right)^{1/2}
	\exp \biggl( - \frac{2 |x_d-t_d|^2}{j_d^2 \gamma_d^2} \biggr)
	f(\boldsymbol{t}) \, dt_d 
	\\
	& = \sum_{j_d=1}^{r_d} {r_d \choose j_d} \frac{(-1)^{1-j_d}}{j_d \gamma_d} 
	\left( \frac{2}{\pi} \right)^{1/2} 
	\int_{\mathbb{R}} \exp \biggl( - \frac{2 |h_d|^2}{\gamma_d^2} \biggr) 
	f_d(x_d + j_d h_d|\boldsymbol{t}) j_d \, dh_d 
	\\
	& = \int_{\mathbb{R}} \frac{1}{\gamma_d} \left( \frac{2}{\pi} \right)^{1/2} 
	\exp \biggl( - \frac{2 |h_d|^2}{\gamma_d^2} \biggr) 
	\biggl( \sum_{j_d=1}^{r_d} {r_d \choose j_d} (-1)^{1-j_d} 
	f_d(x_d + j_d h_d|\boldsymbol{t}) \biggr) \, dh_d,
	\end{align*} 
	where the functions $f_i$'s, $1 \leq i \leq d$, are defined as in \eqref{UnivariateFunction}.
	Recursively we obtain
	\begin{align*}
	K * f (\boldsymbol{x}) 
	= \int_{\mathbb{R}^d} k(\boldsymbol{h})
	\Sigma^{\boldsymbol{r}}
	f(\boldsymbol{x} + \boldsymbol{j} \boldsymbol{h})
	\, d\boldsymbol{h},
	\end{align*} 
	where $k : \mathbb{R}^d \to \mathbb{R}$ is defined by 
	$k(\boldsymbol{h}) := \prod_{i=1}^d k_i(h_i)$ with
	$k_i : \mathbb{R} \to \mathbb{R}$ defined by
	\begin{align*}
	k_i(h_i) := \frac{1}{\gamma_i} \left( \frac{2}{\pi} \right)^{1/2} 
	\exp \biggl( - \frac{2 |h_i|^2}{\gamma_i^2} \biggr),
	\,\,\,\,\,\,\,\,\,\,\,
	1 \leq i \leq d,
	\end{align*} 
	$\Sigma^{\boldsymbol{r}} := \Sigma^{r_1} \cdots \Sigma^{r_d}$ is defined by
	\begin{align*}
	\Sigma^{r_i} := \sum_{j_i=1}^{r_i} {r_i \choose j_i} (-1)^{1-j_i},
	\,\,\,\,\,\,\,\,\,\,\,
	1 \leq i \leq d,
	\end{align*}
	and $\boldsymbol{j} = (j_1, \ldots, j_d)$. Now, for $1 \leq i \leq d-1$, let $\boldsymbol{u}^{(i)} = (u_1^{(i)}, \ldots, u_d^{(i)}) \in \mathbb{R}^d$ denote the vector with
	\begin{align*}
	u_{\ell}^{(i)} =
	\begin{cases}
	0 & \text{ for } \ell = 1, \ldots, i \\
	1 & \text{ for } \ell = i+1, \ldots, d.  
	\end{cases}
	\end{align*}
	Then, for $0 \leq i \leq d$, setting
	\begin{align*}
	\boldsymbol{x}^{(i)} =
	\begin{cases}
	\boldsymbol{x} & \text{ for } i = 0 
	\\
	\boldsymbol{u}^{(i)} \cdot \boldsymbol{x} 
	+ (1 - \boldsymbol{u}^{(i)}) \cdot (\boldsymbol{x} + \boldsymbol{j} \boldsymbol{h}) 
	& \text{ for } i = 1, \ldots, d-1   
	\\
	\boldsymbol{x} + \boldsymbol{j} \boldsymbol{h} & \text{ for } i = d,
	\end{cases}
	\end{align*}
	we obtain
	\begin{align*}
	f(\boldsymbol{x} + \boldsymbol{j} \boldsymbol{h}) - f(\boldsymbol{x})
	= \sum_{i=1}^d f(\boldsymbol{x}^{(i)}) - f(\boldsymbol{x}^{(i-1)})
	= \sum_{i=1}^d f_i(x_i + j_i h_i | \boldsymbol{x}^{(i)}) 
	- f_i(x_i | \boldsymbol{x}^{(i)}).
	\end{align*}
	With this and $\int_{\mathbb{R}^d} k(\boldsymbol{h}) \, d\boldsymbol{h} = 1$
	we obtain, for $q > 1$, 
	\begin{align*}
	\|K * f - f\|_{L_q(\mathrm{P}_X)}^q 
	& = \int_X |K * f(\boldsymbol{x}) - f(\boldsymbol{x})|^q \, d\mathrm{P}_X(\boldsymbol{x}) 
	\\
	& = \int_{\mathbb{R}^d} \biggl| 
	\int_{\mathbb{R}^d} k(\boldsymbol{h})
	\Sigma^{\boldsymbol{r}}
	f(\boldsymbol{x} + \boldsymbol{j} \boldsymbol{h})
	\, d\boldsymbol{h} - f (\boldsymbol{x}) 
	\biggr|^q \, d\mathrm{P}_X(\boldsymbol{x}) 
	\\
	& = \int_{\mathbb{R}^d} \biggl| 
	\int_{\mathbb{R}^d} 
	k(\boldsymbol{h}) \Bigl(
	\Sigma^{\boldsymbol{r}}
	f(\boldsymbol{x} + \boldsymbol{j} \boldsymbol{h})
	- f (\boldsymbol{x}) \Bigr)
	\, d\boldsymbol{h} 
	\biggr|^q \, d\mathrm{P}_X(\boldsymbol{x}).
	\end{align*}
	Next, H\"{o}lder's inequality and $\int_{\mathbb{R}^d} k(\boldsymbol{h}) \, d\boldsymbol{h} = 1$ yield, for $q > 1$,
	\begin{align*}
	\|K * f - f\|_{L_q(\mathrm{P}_X)}^q 
	& \leq \int_{\mathbb{R}^d} \biggl( 
	\biggl( 
	\int_{\mathbb{R}^d} 
	k(\boldsymbol{h}) \, d\boldsymbol{h} 
	\biggr)^{\frac{q-1}{q}}
	\biggl( 
	\int_{\mathbb{R}^d} k(\boldsymbol{h}) 
	\Bigl|
	\Sigma^{\boldsymbol{r}}
	f(\boldsymbol{x} + \boldsymbol{j} \boldsymbol{h})
	- f (\boldsymbol{x}) \Bigr|^q
	\, d\boldsymbol{h} 
	\biggr)^{\frac{1}{q}} \biggr)^q
	\, d\mathrm{P}_X(\boldsymbol{x})
	\\
	& \leq \int_{\mathbb{R}^d} 
	\int_{\mathbb{R}^d} k(\boldsymbol{h}) 
	\biggl|
	\Sigma^{\boldsymbol{r}}
	\sum_{i=1}^d f_i(x_i + j_i h_i | \boldsymbol{x}^{(i)}) 
	- f_i(x_i | \boldsymbol{x}^{(i)}) \biggr|^q
	\, d\boldsymbol{h} 
	\, d\mathrm{P}_X(\boldsymbol{x})
	\\
	& = \int_{\mathbb{R}^d} 
	\int_{\mathbb{R}^d} k(\boldsymbol{h}) 
	\biggl|
	\sum_{i=1}^d \Sigma^{\boldsymbol{r}} \Bigl( f_i(x_i + j_i h_i | \boldsymbol{x}^{(i)}) 
	- f_i(x_i | \boldsymbol{x}^{(i)}) \Bigr) \biggr|^q
	\, d\boldsymbol{h} 
	\, d\mathrm{P}_X(\boldsymbol{x})
	\\
	& = \int_{\mathbb{R}^d} 
	\int_{\mathbb{R}^d} k(\boldsymbol{h}) 
	\biggl|
	\sum_{i=1}^d \Sigma^{\boldsymbol{r} \setminus r_i} \Bigl( \Sigma^{r_i} \bigl( f_i(x_i + j_i h_i | \boldsymbol{x}^{(i)}) 
	- f_i(x_i | \boldsymbol{x}^{(i)}) \bigr) \Bigr) \biggr|^q
	\, d\boldsymbol{h} 
	\, d\mathrm{P}_X(\boldsymbol{x}),
	\end{align*}
	where
	\begin{align*}
	\Sigma^{\boldsymbol{r} \setminus r_i}
	:= \Sigma^{r_1} \cdots \Sigma^{r_{i-1}} \Sigma^{r_{i+1}} \cdots \Sigma^{r_d}.
	\end{align*}
	Since, for $s \geq 0$ and an integer $i \geq 0$, the function $s \to s^i$ is convex, we have for every integer $i \geq 0$ the transformation
	\begin{align*}
	\biggl( \sum_{j=1}^d a_j \biggr)^i 
	= d^i \biggl( \frac{1}{d} \sum_{j=1}^d a_j \biggr)^i 
	\leq d^i \sum_{j=1}^d  \frac{1}{d} a_j^i 
	= d^{i-1} \sum_{j=1}^d a_j^i. 
	\end{align*}
	This leads to
	\begin{align}
	\|K * f - f\|_{L_q(\mathrm{P}_X)}^q 
	& \leq \sum_{i=1}^d \int_{\mathbb{R}^d} 
	\int_{\mathbb{R}^d} k(\boldsymbol{h}) 
	\Bigl|
	\Sigma^{\boldsymbol{r} \setminus r_i} \Bigl( \Sigma^{r_i} \bigl( f_i(x_i + j_i h_i | \boldsymbol{x}^{(i)}) 
	- f_i(x_i | \boldsymbol{x}^{(i)}) \bigr) \Bigr) \Bigr|^q
	\, d\boldsymbol{h} 
	\, d\mathrm{P}_X(\boldsymbol{x})
	\nonumber\\
	& = \sum_{i=1}^d \int_{\mathbb{R}^d} k(\boldsymbol{h}) 
	\int_{\mathbb{R}^d} 
	\Bigl|
	\Sigma^{\boldsymbol{r} \setminus r_i} \Bigl( \Sigma^{r_i} \bigl( f_i(x_i + j_i h_i | \boldsymbol{x}^{(i)}) 
	- f_i(x_i | \boldsymbol{x}^{(i)}) \bigr) \Bigr) \Bigr|^q
	\, d\mathrm{P}_X(\boldsymbol{x})
	\, d\boldsymbol{h} 
	\nonumber\\
	& \leq \sum_{i=1}^d \int_{\mathbb{R}^d} k(\boldsymbol{h}) 
	\prod_{\ell=1,\atop \ell \neq i}^d (2^{r_{\ell}}-1)^q
	\sup_{\boldsymbol{x} \in \mathbb{R}^d}
	\|\triangle_{h_i}^{r_i} 
	\bigl( f_i(\cdot | \boldsymbol{x}), \cdot \bigr)\|_{L_q(\mathrm{P}_{X_i})}^q
	\, d\boldsymbol{h} 
	\nonumber\\
	& \leq \sum_{i=1}^d 
	\int_{\mathbb{R}^d} k(\boldsymbol{h}) 
	\prod_{\ell=1,\atop \ell \neq i}^d (2^{r_{\ell}}-1)^q
	\omega_{r_i,L_q(\mathrm{P}_{X_i})}^q(f_i, h_i) 
	\, d\boldsymbol{h} 
	\nonumber\\
	& \leq \sum_{i=1}^d \prod_{i=1}^d (2^{r_i}-1)^q
	\int_{\mathbb{R}} 
	k_i(h_i) \,
	\omega_{r_i,L_q(\mathrm{P}_{X_i})}^q(f_i, h_i) 
	\, dh_i \ .
	\label{Zwischenresultat}
	\end{align}
	Furthermore, similarly for $q = 1$, we can show that \eqref{Zwischenresultat}
	holds. Consequently, \eqref{Zwischenresultat} holds for all $q \geq 1$. 
	Moreover, we have
	\begin{align*}
	\omega_{r_i, L_q(\mathrm{P}_{X_i})}^q(f_i, t_i) 
	& = \sup_{\boldsymbol{x} \in \mathbb{R}^d} \sup_{|h_i| \leq t_i}  
	\int_{\mathbb{R}} |\triangle_{h_i}^{r_i}(f_i, x_i)|^q \, d\mathrm{P}_{X_i}(x_i) 
	  = \sup_{\boldsymbol{x} \in \mathbb{R}^d} \sup_{|h_i| \leq t_i} 
	\int_{\mathbb{R}} |\triangle_{h_i}^{r_i}(f_i, x_i)|^q g_i(x_i) \, d\mu(x_i) 
	\\ 
	& = \sup_{\boldsymbol{x} \in \mathbb{R}^d} \sup_{|h_i| \leq t_i} 
	\int_{\mathbb{R}} \bigl| \triangle_{h_i}^{r_i}(f_i, x_i) (g_i(x_i))^{1/q} \bigr|^q \, d\mu(x_i) 
	  = \sup_{\boldsymbol{x} \in \mathbb{R}^d} \sup_{|h_i| \leq t_i} 
	\bigl\| \triangle_{h_i}^{r_i}(f_i, \cdot) g_i^{1/q} \bigr\|_{L_q(\mathbb{R})}^q  
	\\
	& \leq \sup_{\boldsymbol{x} \in \mathbb{R}^d} \sup_{|h_i| \leq t_i} 
	\Bigl( \|\triangle_{h_i}^{r_i}(f_i, \cdot)\|_{L_{qs}(\mathbb{R})} \bigl\|g_i^{1/q}\bigr\|_{L_{qp}(\mathbb{R})} \Bigr)^q  
	   = \|g_i\|_{L_p(\mathbb{R})} \omega_{r_i,L_{qs}(\mathrm{P}_{X_i})}^q(f_i, t_i) 
	\\
	& \leq \|g_i\|_{L_p(\mathbb{R})} \left( 1 + \frac{2 t_i}{\gamma_i} \right)^{r_iq} \omega_{r_i,L_{qs}(\mathbb{R})}^q \biggl(f_i,\frac{\gamma_i}{2}\biggr)
	\end{align*}
	where we used \eqref{EigMod_1}. Together with \eqref{Zwischenresultat} this implies
	\begin{align}
	\|K * f - f\|_{L_q(\mathrm{P}_X)}^q 
	& \leq \prod_{i=1}^d (2^{r_i} - 1)^q
	\sum_{i=1}^d
	\int_{\mathbb{R}} 
	k_i(h_i) 
	\|g_i\|_{L_p(\mathbb{R})} \left( 1 + \frac{2 h_i}{\gamma_i} \right)^{r_iq} \omega_{r_i,L_{qs}(\mathbb{R})}^q
	\biggl(f_i,\frac{\gamma_i}{2}\biggr) \, dh_i 
	\nonumber\\
	& = \prod_{i=1}^d (2^{r_i} - 1)^q
	\sum_{i=1}^d
	\|g_i\|_{L_p(\mathbb{R})} \omega_{r_i,L_{qs}(\mathbb{R})}^q
	\biggl( f_i,\frac{\gamma_i}{2} \biggr)
	\int_{\mathbb{R}} 
	k_i(h_i)
	\left( 1 + \frac{2 h_i}{\gamma_i} \right)^{r_iq}  
	\, d h_i. 
	\label{MainEstimate}
	\end{align}
%	Let $\boldsymbol{N} = (N_1, \ldots, N_d) \in \mathbb{N}^d$ with $2 N_i \leq r_i q < 2 (N_i+1)$.
%	Then for all $1 \leq i \leq d$, we have\
    Since we have
	\begin{align*}
	\left( 1 + \frac{2 h_i}{\gamma_i} \right)^{r_i q} 
	\leq \left( 1 + \frac{2 h_i}{\gamma_i} \right)^{\lceil r_i q \rceil}
	\leq \sum_{j=0}^{\lceil r_i q \rceil} \binom{\lceil r_i q \rceil}{j} \left(\frac{2 h_i}{\gamma_i} \right)^j,
	\end{align*}
	therefore,
	\begin{align} \label{inequality_1}
	\left( 1 + \frac{2 h_i}{\gamma_i} \right)^{r_i q} 
	\leq \sum_{j=0}^{\lceil r_i q \rceil} \binom{\lceil r_i q \rceil}{j} \left(\frac{2}{\gamma_i} \right)^j
	\int_{\mathbb{R}^d}{ h_i } ^j k_i(h_i) \,  d h_i
	\leq \sum_{j=0}^{\lceil r_i q \rceil} \binom{\lceil r_i q \rceil}{j} \left(\frac{2}{\gamma_i} \right)^j \bigg( \int_{\mathbb{R}}{ h_i } ^{2j} k_i(h_i) \,  d h_i \bigg)^{\frac{1}{2}},
	\end{align}
	where the last inequality holds with the help of H\"{o}lder inequality and $\int_{\mathbb{R}} k_i(h_i)\, dh_i = 1$.
	For a fixed $i \in \{ 1, \ldots, d \}$,
	with the substitution $h_i = (\frac{\gamma_i^2}{2} u)^{1/2}$, 
	the functional equation $\Gamma(z+1) = z \, \Gamma(z)$ of the Gamma function $\Gamma$, 
	and $\Gamma ( \frac{1}{2} ) = \sqrt{\pi}$ we have
	\begin{align} \label{inequality_2}
	\int_{\mathbb{R}} h_i^{2j}  k_i(h_i) \, d h_i
	= \frac{1}{\gamma_i} \biggl( \frac{2}{\pi} \biggr)^{\frac{1}{2}} 
	      \int_{\mathbb{R}} h_i^{2j} \exp \biggl( - \frac{2 h_i^2}{\gamma_i^2} \biggr) \, d h_i
	= \frac{1}{\gamma_i}  \biggl( \frac{2}{\pi} \biggr)^{\frac{1}{2}} 
	      \biggl( \frac{1}{2} \frac{\gamma_i}{\sqrt{2}} 
	           \biggl( \frac{\gamma_i^2}{2} \biggr)^j \sqrt{\pi} 
	               \prod_{m=1}^j \biggl( m - \frac{1}{2} \biggr) \biggr).
	\end{align}
	Together, \eqref{inequality_1} and \eqref{inequality_2} lead to
	\begin{align*}
	\int_{\mathbb{R}} k_i (h_i) \left( 1 + \frac{2 h_i}{ \gamma_i } \right)^{r_i q}
	& \leq \sum_{j=0}^{\lceil{ r_i q \rceil}} \binom{\lceil r_i q \rceil}{j} \left(\frac{2}{\gamma_i} \right)^j \left( \left( \frac{{\gamma_i}^2}{2} \right)^j \frac{1}{2} \prod_{m=1}^j (m-\frac{1}{2})\right)^{\frac{1}{2}}
	\\
	& = \sum_{j=0}^{\lceil{ r_i q \rceil}} \binom{\lceil r_i q \rceil}{j} 2^{\frac{j-1}{2}} \left( \prod_{m=1}^j (m - \frac{1}{2}) \right)^{\frac{1}{2}}
    \triangleq c'_{r_i,\, q}.
	\end{align*}
	According to \eqref{MainEstimate}, we have
	\begin{align*}
	\|K * f - f\|_{L_q(\mathrm{P}_X)}^q 
	\leq \sum_{i=1}^d c_{r_i,\, q} \| g_i \|_{L_q(\mathbb{R})} \omega_{r_i,L_{qs}(\mathbb{R})}^q \left(f_i, \frac{\gamma_i}{2} \right)
	\end{align*}
	where $c_{r_i,\, q} := c'_{r_i,\, q} \prod_{i=1}^d ( 2^{r_i} - 1 )^q$.
	
\end{proof}

\begin{proof}[Proof of Proposition \ref{ConvolutionInRKHS}]
	Recall that, for 
	$\boldsymbol{r} = (r_1, \ldots, r_d) \in \mathbb{N}^d$,
	$\boldsymbol{\gamma} = (\gamma_1, \ldots, \gamma_d) \in \mathbb{R}_+^d$,
	and $\boldsymbol{x} = (x_1, \ldots, x_d) \in \mathbb{R}^d$,
	the function $K : \mathbb{R}^d \to \mathbb{R}$ is defined by
	\begin{align*}
	K(\boldsymbol{x}) := \prod_{i=1}^d K_i(x_i)
	\end{align*}
	with functions $K_i : \mathbb{R} \to \mathbb{R}$ defined by
	\begin{align*}
	K_i(x_i) := \left( \frac{2}{\pi} \right)^{1/2} 
	\sum_{j_i=1}^{r_i} {r_i \choose j_i} \frac{(-1)^{1-j_i}}{j_i \gamma_i} 
	\exp \biggl( - \frac{2 |x_i|^2}{j_i^2 \gamma_i^2} \biggr),
	\,\,\,\,\,\,\,\,
	1 \leq i \leq d.
	\end{align*}
	We define, for all $1 \leq i \leq d$,
	\begin{align}\label{hat_K_j_i}
	\widehat{K}_{j_i}(x_i) 
	:= \left( \frac{2}{\pi} \right)^{1/4}
	\left( \frac{1}{j_i \gamma_i} \right)^{1/2} 
	\exp \biggl( - \frac{2 |x_i|^2}{j_i^2 \gamma_i^2} \biggr),
	\,\,\,\,\,\,\,\,
	x_i \in X_i.
	\end{align}
	By Proposition 4.46 in \cite{StCh08}, we obtain
	\begin{align*}
	\widehat{K}_{j_i} * f_i(\cdot | \boldsymbol{x}) \in H_{j_i \gamma_i}(X_i) \subset H_{\gamma_i}(X_i) 
	\end{align*}
	for all $j_i \in \mathbb{N}$, where the functions $f_i(\cdot | \boldsymbol{x})$'s, $1 \leq i \leq d$, are defined as in \eqref{UnivariateFunction}.
	Due to the properties of the convolution, we finally obtain
	\begin{align*}
	K_i * f_i(\cdot | \boldsymbol{x}) 
	= \sum_{j_i=1}^{r_i} {r_i \choose j_i} (-1)^{1-j_i} 
	\left( \frac{2}{\pi} \right)^{1/4}
	\left( \frac{1}{j_i \gamma_i} \right)^{1/2} 
	(\widehat{K}_{j_i} * f_i(\cdot | \boldsymbol{x})) \in H_{\gamma_i}(X_i) \ .
	\end{align*}
	Now, recall the RKHS norm \eqref{RKHS_Norm} which is defined as
	\begin{align*}
	\|f\|_{H_{\boldsymbol{\gamma}}}
	:= \biggl( \frac{2}{\pi} \biggr)^{d/2} 
	\prod_{i=1}^d \gamma_i^{-1} 
	\biggl( \int_{\mathbb{R}^d} |f(\boldsymbol{x})|^2 
	\exp \biggl( - \sum_{i=1}^d \frac{4 x_i^2}{\gamma_i^2} \biggr)
	\, d\boldsymbol{x} \biggr)^{1/2},
	\end{align*}
	an elementary caculation shows that the following equation holds
	\begin{align*}
	\|K * f\|_{H_{\boldsymbol{\gamma}}}^2
	= \biggl\| \biggl( \prod_{i=1}^d K_i \biggr) * f \biggr\|_{H_{\gamma_1} \otimes \ldots \otimes H_{\gamma_d}}^2
    &  = \biggl\| \biggl( \prod_{i=1}^{d-1} K_i \biggr) * (K_d * f_d(\cdot | \boldsymbol{x})) \biggr\|_{H_{\gamma_1} \otimes \ldots \otimes H_{\gamma_d}}^2
	\\
	& = \biggl\| \biggl( \prod_{i=1}^{d-1} K_i \biggr) * \big\| K_d * f_d (\cdot | \boldsymbol{x}) \big\|_{H_{\gamma_d}} \biggr\|_{H_{\gamma_1} \otimes \ldots \otimes H_{\gamma_{d-1}}}^2.
	\end{align*}
	Moreover, the definition of the RKHS norm \eqref{RKHS_Norm} also implies
	\begin{align*}
	\|K_d * f_d(\cdot | \boldsymbol{x})\|_{H_{\gamma_d}}
	& \leq \sum_{j_d=1}^{r_d} j_d^{1/2} 
	\biggl\| {r_d \choose j_d} (-1)^{1-j_d} 
	\biggl( \frac{2}{j_d^2 \gamma_d^2 \pi} \biggr)^{1/2}
	\exp \biggl( - \frac{2 |\cdot|^2}{j_d^2 \gamma_d^2} \biggr) * f_d(\cdot | \boldsymbol{x}) \biggr\|_{H_{j_d \gamma_d}} 
	\\
	& \leq \sum_{j_d=1}^{r_d} j_d^{1/2} {r_d \choose j_d} 
	\biggl( \frac{1}{j_d \gamma_d \sqrt{\pi}} \biggr)^{1/2} \|f_d(\cdot | \boldsymbol{x})\|_{L_2(\mathbb{R})} 
	\\
	& = \biggl( \frac{1}{\gamma_d \sqrt{\pi}} \biggr)^{1/2} (2^{r_d} - 1)       
	\|f_d(\cdot | \boldsymbol{x})\|_{L_2(\mathbb{R})} \ ,
	\end{align*}
	where we used Proposition 4.46 in \cite{StCh08} in the first two steps. 
	Recursively we obtain
	\begin{align*}
	\|K * f\|_{H_{\boldsymbol{\gamma}}} 
	& = \biggl\| \biggl( \prod_{i=1}^{d-1} K_i \biggr) * \big\| K_d * f_d (\cdot | \boldsymbol{x}) \big\|_{H_{\gamma_d}} \biggr\|_{H_{\gamma_1} \otimes \ldots \otimes H_{\gamma_{d-1}}}
	\\
	& \leq 
	\biggl( \frac{1}{\gamma_d \sqrt{\pi}} \biggr)^{1/2} (2^{r_d} - 1)   
	\biggl\| \biggl( \prod_{i=1}^{d-1} K_i \biggr) *    
	\|f_d(\cdot | \boldsymbol{x})\|_{L_2(\mathbb{R})}
	\biggr\|_{H_{\gamma_1} \otimes \ldots \otimes H_{\gamma_{d-1}}}
	\\
	& \leq \biggl( \prod_{i=1}^d \gamma_i \biggr)^{-1/2} \pi^{-d/4} 
	\prod_{i=1}^d (2^{r_i} - 1) \|f\|_{L_2(\mathbb{R}^d)} \ .
	\end{align*}
	Finally, for all $\boldsymbol{x} \in X$ and $g \in L_{\infty}(\mathbb{R}^d)$, H\"{o}lder's inequality implies
	\begin{align*}
	|K * f(\boldsymbol{x})| 
	& \leq \sup_{\hat{\boldsymbol{x}} \in X} |K * f(\hat{\boldsymbol{x}})|
      \leq \sup_{\hat{\boldsymbol{x}} \in X} 
	\int_{\mathbb{R}^d} |K(\hat{\boldsymbol{x}} - \boldsymbol{t}) f(\boldsymbol{t})| \, d\boldsymbol{t} 
	\\
	& \leq \|f\|_{L_{\infty}(\mathbb{R}^d)} 
	\prod_{i=1}^d
	\sum_{j_i=1}^{r_i} {r_i \choose j_i} 
	\sup_{\hat{x}_i \in X_i} \int_{\mathbb{R}} \left( \frac{2}{j_i^2 \gamma_i^2 \pi} \right)^{1/2}
	\exp \biggl( - \frac{2 |\hat{x}_i - t_i|^2}{(j_i \gamma_i)^2} \biggr) \, dt_i 
	= \prod_{i=1}^d (2^{r_i} - 1) \|f\|_{L_{\infty}(\mathbb{R}^d)} \ .
	\end{align*}
\end{proof}

\begin{proof}[Proof of Theorem \ref{ApproximationError}]
	First, the assumption $f^*_{L,\mathrm{P}} \in L_2(\mathbb{R}^d) \cap L_{\infty}(\mathbb{R}^d)$ and Proposition \ref{ConvolutionInRKHS} immediately yield
	\begin{align*}
	f_0 = K * f_{L,\mathrm{P}}^* \in H_{\boldsymbol{\gamma}}.
	\end{align*}
	Furthermore, because of $f^*_{L,\mathrm{P}} \in L_{\infty}(\mathbb{R}^d)$
	and Proposition \ref{ConvolutionInRKHS}, the estimate
	\begin{align*}
	|K * f_{L,\mathrm{P}}^*(\boldsymbol{x})|
	\leq \prod_{i=1}^d (2^{r_i} - 1) \|f_{L,\mathrm{P}}^*\|_{L_{\infty}(\mathbb{R}^d)} 
	\end{align*}
	holds for all $\boldsymbol{x} \in X$. 
	This implies, for all $(x, y) \in X \times Y$,
	\begin{align*}
	L\circ f_0 =L(y, K * f_{L,\mathrm{P}}^*(\boldsymbol{x}))
	& = (y - K * f_{L,\mathrm{P}}^*(\boldsymbol{x}))^2
      = y^2 - 2 y (K * f_{L,\mathrm{P}}^*(\boldsymbol{x}))
	+ (K * f_{L,\mathrm{P}}^*(\boldsymbol{x}))^2
	\\
	& \leq M^2 + 2 M \prod_{i=1}^d (2^{r_i} - 1)       
	\|f_{L,\mathrm{P}}^*\|_{L_{\infty} (\mathbb{R}^d)} 
	+ \prod_{i=1}^d (2^{r_i} - 1)^2       
	\|f_{L,\mathrm{P}}^*\|_{L_{\infty} (\mathbb{R}^d)}^2 
	\\
	& \leq \prod_{i=1}^d 4^{r_i} 
	\max \{ M, \|f_{L,\mathrm{P}}^*\|_{L_{\infty} (\mathbb{R}^d)} \}^2
	\end{align*}
	and
	\begin{align*}
	\|L \circ f_0\|_{\infty}
	= \sup_{(x, y) \in X \times Y} |L(y, f_0(\boldsymbol{x}))|
	= \sup_{(x, y) \in X \times Y} |L(y, K * f_{L,\mathrm{P}}^*(\boldsymbol{x}))|
	\leq \prod_{i=1}^d 4^{r_i} 
	\max \{ M, \|f_{L,\mathrm{P}}^*\|_{L_{\infty} (\mathbb{R}^d)} \}^2 =: B_0.
	\end{align*}
	Furthermore, \eqref{ExcessRiskLS} and Proposition \ref{ErrorEstimation} yield
	\begin{align*}
	\mathcal{R}_{L,\mathrm{P}}(f_0) - \mathcal{R}_{L,\mathrm{P}}^* 
	& = \mathcal{R}_{L,\mathrm{P}}(K * f_{L,\mathrm{P}}^*) 
	- \mathcal{R}_{L,\mathrm{P}}^* 
	= \|K * f_{L,\mathrm{P}}^* - f^*_{L,\mathrm{P}}\|_{L_2(\mathrm{P}_X)}^2 
	\\
	& \leq \sum_{i=1}^d c_{r_i,2} \|g_i\|_{L_q(\mathbb{R})} 
	\omega_{r_i,L_{2s}(\mathbb{R})}^2 
	(f_{L,\mathrm{P},i}^*, \gamma_i/2)
	\leq \sum_{i=1}^d c_{r_i,2} \|g_i\|_{L_q(\mathbb{R})} \, c_i^2 \gamma_i^{2\alpha_i} \ ,
	\end{align*}
	where we used
	$\omega_{r_i,L_{2s}(\mathbb{R})} (f_{L,\mathrm{P},i}^*, \gamma_i/2) \leq c_i \gamma_i^{\alpha_i}$
	for $\boldsymbol{\gamma} = (\gamma_1, \ldots, \gamma_d) \in \mathbb{R}_+^d$,
	$\boldsymbol{\alpha} = (\alpha_1, \ldots, \alpha_d) \in [1, \infty)^d$,
	$\boldsymbol{r} = (r_1, \ldots, r_d) \in \mathbb{N}^d$,
	$r_i = \lfloor \alpha_i \rfloor + 1$ and constants $c_i > 0$, $1 \leq i \leq d$, in the last step, which in turn immediately results from the assumption $f_{L,\mathrm{P}}^* \in B_{2s,\infty}^{\boldsymbol{\alpha}}(\mathbb{R}^d)$. By Theorem \ref{ConvolutionInRKHS}
	we know
	\begin{align*}
	\|f_0\|_{H_{\boldsymbol{\gamma}}} 
	= \|K * f_{L,\mathrm{P}}^*\|_{H_{\boldsymbol{\gamma}}} 
	\leq  \pi^{-\frac{d}{4}} 
	\prod_{i=1}^d (2^{r_i} - 1) \|f_{L,\mathrm{P}}^*\|_{L_2(\mathbb{R}^d)} 
	\prod_{i=1}^d \gamma_i^{-\frac{1}{2}} \ .
	\end{align*}
	Since we assumed $p, s \geq 1$ with $\frac{1}{p} + \frac{1}{s} = 1$, i.e.~the marginal
	densities $g_i$ of $\mathrm{P}_{X_i}$ is contained in $L_p(X_i)$,
	the above discussion 
	together with \eqref{ExcessRiskLS}
	yields
	\begin{align}
	\min_{f \in H_{\boldsymbol{\gamma}}} \lambda \|f\|_{H_{\boldsymbol{\gamma}}}^2 + \mathcal{R}_{L,\mathrm{P}}(f) - \mathcal{R}_{L,\mathrm{P}}^* 
	& \leq \lambda \|f_0\|_{H_{\boldsymbol{\gamma}}}^2 + \mathcal{R}_{L,\mathrm{P}}(f_0) - \mathcal{R}_{L,\mathrm{P}}^* 
	\nonumber\\
	& = \lambda \|K * f_{L,\mathrm{P}}^*\|_{H_{\boldsymbol{\gamma}}}^2 + \mathcal{R}_{L,\mathrm{P}}(K * f_{L,\mathrm{P}}^*) - \mathcal{R}_{L,\mathrm{P}}^*  \nonumber\\
	& \leq \pi^{-\frac{d}{2}} \prod_{i=1}^d (2^{r_i} - 1)^2      
	\|f_{L,\mathrm{P}}^*\|_{L_2(\mathbb{R}^d)}^2 
	\lambda \prod_{i=1}^d \gamma_i^{-1}
	+ \|K * f_{L,\mathrm{P}}^* - f_{L,\mathrm{P}}^*\|_{L_2(\mathrm{P}_X)}^2 
	\nonumber\\
	& \leq \pi^{-\frac{d}{2}} \prod_{i=1}^d (2^{r_i} - 1)^2
	\|f_{L,\mathrm{P}}^*\|_{L_2(\mathbb{R}^d)}^2 
	\lambda \prod_{i=1}^d \gamma_i^{-1} 
	+ \sum_{i=1}^d c_{r_i,2} \|g_i\|_{L_p(\mathbb{R})} \, \omega_{r_i,L_{2 s}(\mathbb{R})}^2 ( f_{L,\mathrm{P},i}^*, \gamma_i/2 ) \ .
	\label{ApproxErrorEstim_1}
	\end{align}
	Now, to further bound \eqref{ApproxErrorEstim_1},
	we have to estimate the modulus of smoothness. 
	To this end, recall that
	$f_{L,\mathrm{P}}^* \in B_{2s,\infty}^{\boldsymbol{\alpha}}(\mathbb{R}^d)$.
	By the definition of $\mathrm{Lip}^*(\alpha_i,L_{2s}(\mathbb{R}))
	= B_{2s,\infty}^{\alpha_i}(\mathbb{R})$,
	we have for all $1 \leq i \leq d$,
	\begin{align*}
	\omega_{r_i,L_{2s}(\mathbb{R})} (f_{L,\mathrm{P},i}^*, t_i) \leq c_{s,i} t_i^{\alpha_i} \ ,
	\,\,\,\,\,\,\,\,
	t_i > 0 \ ,
	\end{align*}
	where $r_i := \lfloor \alpha_i \rfloor + 1$ and $c_{s,i} > 0$ are suitable constants. Using this inequality the upper
	bound of the approximation error only depends on the kernel width $\boldsymbol{\gamma}$, the
	regularization parameter $\lambda$, the smoothness parameter $\boldsymbol{\alpha}$ of the target function
	and some positive constants, i.e.
	\begin{align*}
	\min_{f \in H_{\boldsymbol{\gamma}}} \lambda \|f\|_{H_{\boldsymbol{\gamma}}}^2 + \mathcal{R}_{L,\mathrm{P}}(f) - \mathcal{R}_{L,\mathrm{P}}^* 
	\leq C_1 \lambda \prod_{i=1}^d \gamma_i^{-1} 
	+ C_s \sum_{i=1}^d \gamma_i^{2 \alpha_i}.
	\end{align*}
\end{proof}

In order to prove the main theorem given in Theorem \ref{LearningRates}, 
we need the following lemma which bounds the constant $K(p)$ defined in \eqref{K_p}.

\begin{lemma} \label{BoundKp}
	For the constant $K(p)$ defined in \eqref{K_p}, there holds
	\begin{align*}
	\max_{p \in (0, \frac{1}{2}]} K(p)
	\leq 3 \cdot 10^8 e^2 M^2.
	\end{align*}
\end{lemma}

\begin{proof}[Proof of Lemma \ref{BoundKp}]
	Here we are interested to bound $K(p)$ for $p \in (0, \frac{1}{2}]$. 
	For this, we first need to bound the constants $C_1(p)$ and $C_2(p)$. 
	We start with $C_p$ and obtain the following bound for $p \in (0, \frac{1}{2}]$,
	\begin{align*}
	C_p^p = \biggl(  \frac{\sqrt{2} - 1}{\sqrt{2} - 2^{\frac{2p-1}{2p}}} \biggr)^p
	\biggl( \frac{1-p}{p} \biggr)^p
	\leq e \max_{p \in (0, \frac{1}{2}]} 
	\biggl(  \frac{\sqrt{2} - 1}{\sqrt{2} - 2^{\frac{2p-1}{2p}}} \biggr)^p
	= e,
	\end{align*}
	where we used $(\frac{1-p}{p})^p = (\frac{1}{p}-1)^p \leq e$
	for all $p \in (0, \frac{1}{2}]$, and Lemma 14 established in \cite{farooq2017svm}. 
	Now the bound for $C_1(p)$ is the following:
	\begin{align*}
	C_1(p) 
	\leq \max_{p \in (0, \frac{1}{2}]} \frac{2 \sqrt{\log 256} C_p^p}{(\sqrt{2} - 1) (1 - p) 2^{p/2}}
	\leq \frac{4 e \sqrt{\log 256}}{\sqrt{2} - 1} \max_{p \in (0, \frac{1}{2}]} \sqrt{1}{2^{p/2}}
	\leq 46 e.
	\end{align*}
	Analogously, the bound for the constant $C_2(p)$ is:
	\begin{align*}
	C_2(p)^{1+p} 
	\leq \max_{p \in (0, \frac{1}{2}]} \biggl( \frac{8 \sqrt{\log 16} C_p^p}{(\sqrt{2} - 1) (1 - p) 4^p} \biggr)^2
	\leq \frac{256 e^2 \log 16}{(\sqrt{2} - 1)^2} \max_{p \in (0, \frac{1}{2}]} \sqrt{1}{4^{2p}}
	\leq 1035 e^2.
	\end{align*}
	By plugging $C_1(p)$ and $C_2(p)$ into \eqref{K_p}, we thus obtain
	\begin{align*}
	K \leq \max \{ 2 \cdot 10^8 e M, 3 \cdot 10^8 e^2 M^2, 8 M^2 \}
	\leq 3 \cdot 10^8 e^2 M^2.
	\end{align*}
\end{proof}

\begin{proof}[Proof of Theorem \ref{LearningRates}]
	By plugging the estimate \eqref{ApproximationErrorEstimation} from Theorem \ref{ApproximationError},
	$a = (3 K)^{\frac{1}{2p}} \left( \frac{d+1}{e p} \right)^{\frac{d+1}{2p}} \left( \prod_{i=1}^d \gamma_i\right)^{-\frac{1}{2p}}$ from Theorem 
	\ref{EntropyNumberRKHS}, and the bound for $K(p)$ on $(0, \frac{1}{2}]$ from Lemma \ref{BoundKp},
	into \eqref{OracleInequality}, we obtain
	\begin{align}
	\lambda \|f_{\mathrm{D},\lambda,\boldsymbol{\gamma}}\|_{H_{\boldsymbol{\gamma}}}^2 
	+ \mathcal{R}_{L,\mathrm{P}}(\wideparen{f}_{\mathrm{D},\lambda,\boldsymbol{\gamma}}) 
	- \mathcal{R}^{*}_{L,\mathrm{P}} 
	& \leq 9 C_1 \lambda \prod_{i=1}^d \gamma_i^{-1} 
	+ 9 C_s \sum_{i=1}^d \gamma_i^{2 \alpha_i}
	+ 3 C_2 \frac{\prod_{i=1}^d \gamma_i^{-1}}{p^{d+1} \lambda^p n} 
	+ C_3 \frac{\varrho}{n}
	\nonumber\\
	& \leq \frac{C M^2}{p^{d+1}} 
	\biggl( \lambda \prod_{i=1}^d \gamma_i^{-1} 
	+ \sum_{i=1}^d \gamma_i^{2 \alpha_i}
	+ \frac{\prod_{i=1}^d \gamma_i^{-1}}{\lambda^p n} 
	+ \frac{\varrho}{n} \biggr),
	\label{OracleInbetween}
	\end{align}
	where $C_1$ and $C_s$ are from Proposition \ref{ApproximationError}, 
	$C_2 := 9 \cdot 10^8 e^2 M^2 K ( \frac{d+1}{e} )^{d+1}$ is a constant only depending on $d$, 
	$C_3 := ( 3456 M^2 + 15 \cdot \prod_{i=1}^d 4^{r_i} 
	\max \{ M, \|f_{L,\mathrm{P}}^*\|_{L_{\infty} (\mathbb{R}^d)} \}^2 ) (1 + \log 3)$,
	and 
	$C$ is a constant independent of $p$, $\lambda$, $\gamma$, $n$ and $\varrho$. 
	Setting 
	\begin{align*}
	g(\gamma_1, \ldots, \gamma_d) 
	:= \lambda \prod_{i=1}^d \gamma_i^{-1} 
	+ \sum_{i=1}^d \gamma_i^{2 \alpha_i}
	+ \biggl( \prod_{i=1}^d \gamma_i \biggr)^{-1}
	\lambda^{-p} n^{- 1},
	\end{align*}
	then $g$ attains its minimum with respect to $\boldsymbol{\gamma} = (\gamma_1, \ldots, \gamma_d)$ if for all $1 \leq j \leq d$,
	\begin{align*}
	\frac{\partial g}{\gamma_j}
	= - \gamma_j^{-1} \lambda \prod_{i=1}^d \gamma_i^{-1} 
	+ 2 \alpha_j \gamma_j^{-1} \gamma_j^{2 \alpha_i}
	- \gamma_j^{-1}
	\biggl( \prod_{i=1}^d \gamma_i \biggr)^{-1}
	\lambda^{-p} n^{- 1} = 0,
	\end{align*}
	which implies for all $1 \leq j \leq d$, there holds
	\begin{align*}
	2 \alpha_j \gamma_j^{2 \alpha_j}
	= \lambda \prod_{i=1}^d \gamma_i^{-1}  +
	\biggl( \prod_{i=1}^d \gamma_i \biggr)^{-1}
	\lambda^{-p} n^{-1}.
	\end{align*}
	Now, let $\gamma_0$ satisfying
	\begin{align*}
	2 \alpha_i \gamma_i^{2 \alpha_i} 
	= \gamma_0^{2 \alpha_0},
	\,\,\,\,\,\,\,\,\,\,\,\,
	i = 1, \ldots, d,
	\end{align*}
	where $\alpha_0 := d (\sum_{i=1}^d \frac{1}{\alpha_i})^{-1}$ is the mean smoothness defined as in
	\eqref{MeanSmoothness}. Then we have
	\begin{align*}
	\prod_{i=1}^d \gamma_i 
	= \prod_{i=1}^d \biggl( \frac{1}{2 \alpha_i} \biggr)^{\frac{1}{2 \alpha_i}} 
	\gamma_0^{\alpha_0 \sum_{i=1}^d \frac{1}{\alpha_i}}
	= \prod_{i=1}^d \biggl( \frac{1}{2 \alpha_i} \biggr)^{\frac{1}{2 \alpha_i}} 
	\gamma_0^d.
	\end{align*}
	Consequently, \eqref{OracleInbetween} becomes
	\begin{align*}
	C M^2 p^{-(d+1)}  \biggl( 
	( \lambda \gamma_0^{-d} 
	+ \gamma_0^{2 \alpha_0} 
	+ \gamma_0^{-d}
	\lambda^{-p} 
	n^{-1}  )
	+ \varrho n^{-1}  
	\biggr).
	\end{align*}
	Now, optimizing over $\varepsilon$ together with some standard techniques, see 
	\cite[Lemmas A.1.6 and A.1.7]{StCh08}, we then see that 
	if we assume $p := 1/\log n$,
	the LS-SVM using anisotropic Gaussian RKHS $H_{\boldsymbol{\gamma}}$ and
	\begin{align*}
	\lambda_n = n^{- 1}
	\qquad
	\text{ and }
	\qquad
	\gamma_{i,n} = n^{- \frac{\alpha_0}{\alpha_i(2 \alpha_0 + d)}}, 
	\, 
	i = 1, \ldots, d,
	\end{align*}
	learns with rate 
	\begin{align*} 
	C M^2 p^{-(d+1)}  \biggl( 
	( \lambda \gamma_0^{-d} 
	+ \gamma_0^{2 \alpha_0} 
	+ \gamma_0^{-d}
	\lambda^{-p} 
	n^{-1}  )
	+ \varrho n^{-1}  
	\biggr)
	\leq C  (\log n)^{d+1} \biggl( 
	n^{- \frac{2 \alpha_0}{2 \alpha_0 + d}} + \varrho n^{-1}  
	\biggr)
	\leq C  (\log n)^{d+1}  
	n^{- \frac{2 \alpha_0}{2 \alpha_0 + d}},
	\end{align*}
	where the positive constant $C$
	is independent of $p$.
\end{proof}

\section{Discussion and Conclusions} \label{sec::DisCon}
In this paper, we investigate the nonparametric regression problem using SVMs with anisotropic Gaussian RBF kernels.
To be specific, by assuming that the target functions are contained in certain anisotropic Besov spaces, we establish the almost optimal learning rates, that is, optimal up to some logarithmic factor, presented by the effective smoothness whose reciprocal is the mean of the reciprocals of smoothness of all dimensions.
With the effective smoothness taken into consideration, our almost optimal learning rates are faster than those obtained with the underlying RKHSs being certain anisotropic Sobolev spaces. Moreover, if we assume that the target function depends only on fewer dimensions, which is often the case in practice, even faster learning rates can be achieved.

%\section*{Acknowledgments}

\section*{Appendix.}

The whole Appendix is dedicated to the proof of the Imbedding Theorem for Anisotropic Besov Spaces. To this end, let us begin with the definition of the anisotropic Sobolev space, see also \cite{schmeisser1976anisotropic}.

\begin{definition}[Anisotropic Sobolev Space]
	Let $\boldsymbol{m} = (m_1, \cdots, m_d)$,
	$\boldsymbol{\alpha} = (\alpha_1, \cdots, \alpha_d) \in \mathbb{N}^d$ be vectors of natural numbers. The anisotropic Sobolev space can be defined as
	\begin{align}
	W^{\boldsymbol{m}}_p := 
	\biggl\{ f \ \bigg| \ D^{\boldsymbol{\alpha}} f \in L_p(\mathbb{R}^d), \, 
	           \sum_{j=1}^d \frac{\alpha_j}{m_j} \leq 1 \biggr\}
	\end{align}
   with the corresponding norm 
	\begin{align*}
	\|f\|_{W^{\boldsymbol{m}, p}} := \sum_{\Sigma \, \alpha_j / m_j \leq 1} 
	\|D^{\boldsymbol{\alpha}} f\|_{L_p}.
	\end{align*}
\end{definition}

In the following, for $\boldsymbol{m} = (m_1, \cdots, m_d) \in \mathbb{N}^d$,
we denote $|\boldsymbol{m}| := \sum_{i =1}^{d}m_i$ and set $\mathrm{dom}\, \boldsymbol{m} := \{\boldsymbol{\alpha}\in \mathbb{N}^d : \sum \alpha_i / m_i \leq  1\}$. Moreover, we define the boundary and interior of $\mathrm{dom}\, \boldsymbol{m}$ as $$
\partial \boldsymbol{m} = \biggl\{\boldsymbol{\alpha}\in \mathbb{N}^d : \sum \alpha_i / m_i \leq  1 \;and\; \exists\, i \leq d\ s.t. \; (\alpha_i + 1) / m_i + \sum_{j \neq i} \alpha_j / m_j  > 1 \biggr\}
$$ 
and $\mathrm{int}\, \boldsymbol{m} := \mathrm{dom}\, \boldsymbol{m}/\partial \boldsymbol{m}$, respectively.

With the above preparations, we now present
the anisotropic Sobolev Imbedding Theorem which is an anisotropic extension of the isotropic Sobolev Imbedding Theorem established in  
\cite[Theorem 4.12]{adams2003sobolev}.

\begin{theorem}[The Anisotropic Sobolev Imbedding Theorem] \label{ASIT}
	Let $\Omega$ be a domain in $\mathbb{R}^d$ and, for $1 \leq k \leq d$, 
	let $\Omega_k$ be the intersection of $\Omega$ with a $k$-dimensional plane in $\mathbb{R}^d$. 
	Let $j \geq 0$ and $m \geq 1$ be integers and let $1 \leq p < \infty$.
	Suppose $\Omega$ satisfies the cone condition.
	
	\textbf{Case A} If either $\underline{\boldsymbol{m}} p > d$, then
	\begin{align} \label{01}
	W^{\boldsymbol{m}}_p(\Omega) \to C_B^0(\Omega). 
	\end{align}
	Moreover, if $1 \leq k \leq d$, then
	\begin{align} \label{2}
	W^{\boldsymbol{m}}_p(\Omega) \to L^{q}(\Omega_k) 
	\qquad \text{ for } 
	p \leq q \leq \infty,
	\end{align}
	and, in particular,
	\begin{align*}
		W^{\boldsymbol{m}}_p(\Omega) \to L^q(\Omega) 
		\qquad \text{ for } 
		p \leq q \leq \infty.
	\end{align*}
	
	\textbf{Case B} If $1 \leq k \leq d$ and $\overline{\boldsymbol{m}} p \geq d$ and $\underline{\boldsymbol{m}} p \leq d$, then
	\begin{align}
	W^{\boldsymbol{m}}_p(\Omega) \to L^{q}(\Omega_k), 
	\qquad \text{ for } 
	p \leq q < \infty,
	\end{align}
	and, in particular,
	\begin{align*}
		W^{\boldsymbol{m}}_p(\Omega) \to L^q(\Omega), 
		\qquad \text{ for } 
		p \leq q < \infty.
	\end{align*}
	
	\textbf{Case C} If $\overline{\boldsymbol{m}} p < d$ and  $d - \underline{\boldsymbol{m} }p < k \leq d$, then

	\begin{align}
	W^{\boldsymbol{m}}_p(\Omega) \to L^{q}(\Omega_k), 
	\qquad \text{ for } 
	p \leq q \leq p^* = k p /(d - \underline{\boldsymbol{m}} p).
	\end{align}
	In particular,
	\begin{align}
	W^{\boldsymbol{m}}_p(\Omega) \to L^q(\Omega), 
	\qquad \text{ for } 
	p \leq q \leq p^* = d p /(d - \underline{\boldsymbol{m}} p).
	\end{align}
	The imbedding constants for the imbeddings above depend only on $d$, $\boldsymbol{m}$, $p$, $q$, $j$, $k$, and the dimensions of the cone $C$ in the cone condition.
	
\end{theorem}

Note that we need the following three Lemmas \ref{ALE}, \ref{lemma3} and \ref{lemma4} to prove Theorem \ref{ASIT}, and they are presented here.

\begin{lemma}[A Local Estimate]  \label{ALE}
	Let domain $\Omega \subseteq \mathbb{R}^d$ satisfy the cone condition. 
	There exists a constant $K$ depending on $m$, $d$, and the dimensions $\rho$ and $\kappa$ of the cone $C$ specified in the cone condition for $\Omega$ such that 
	for every $u \in C^{\infty}(\Omega)$, every $x \in \Omega$, and every $r$ satisfying 
	$0 < r \leq \rho$, we have
	\begin{align} \label{localEs}
	|u(x)|
	\leq K \biggl( \sum_{\boldsymbol{\alpha} \in \mathrm{int}\, \boldsymbol{m}} r^{|\boldsymbol{\alpha}|-d} \int_{C_{x,r}} |D^{\boldsymbol{\alpha}} u(y)| \, dy
	+ \sum_{\boldsymbol{\alpha} \in \partial\boldsymbol{m}} \int_{C_{x,r}} |D^{\boldsymbol{\alpha}}u(y)| |x - y|^{|\boldsymbol{\alpha}|-d} \, dy \biggr) \ ,
	\end{align}
	where $C_{x,r} = \{ y \in C_x : |x - y|\leq r \}$. 
	Here $C_x \subseteq \Omega$ is a cone congruent to $C$ having vertex at $x$.
\end{lemma}

\begin{proof}[Proof of Lemma \ref{ALE}]
	We apply Taylor's formula with integral remainder,
	\begin{align*}
		f(1) = \sum_{\boldsymbol{\alpha} \in \mathrm{int}\, \boldsymbol{m}} \frac{1}{|\boldsymbol{\alpha}|!} f^{(\boldsymbol{\alpha})}(0)
		+ \sum_{\boldsymbol{\alpha} \in \partial \boldsymbol{m}}\frac{1}{(|\boldsymbol{\alpha}|-1)!} \int_0^1 (1 - t)^{|\boldsymbol{\alpha}|-1} f^{(\alpha)}(t) \, dt
	\end{align*}
	to the function $f(t) = u(t x + (1-t) y)$, 
	where $x \in \Omega$ and $y \in C_{x,r}$. We note that
	\begin{align*}
		f^{(\boldsymbol{\alpha})}(t) =  \frac{|\boldsymbol{\alpha}|!}{\boldsymbol{\alpha}!}D^{\boldsymbol{\alpha} }u(t x + (1-t) y) (x - y)^{\boldsymbol{\alpha}},
	\end{align*}
	where $\boldsymbol{\alpha}! = \alpha_1 ! \cdots \alpha_d !$ and 
	$(x - y)^{\boldsymbol{\alpha}} = (x_1 - y_1)^{\alpha_1} \cdots (x_d - y_d)^{\alpha_d}$, we obtain
		\begin{align*}
		\begin{split}
			|u(x)|
			\leq \sum_{\boldsymbol{\alpha} \in \mathrm{int}\, \boldsymbol{m}} \frac{1}{\boldsymbol{\alpha}!}   |D^{\boldsymbol{\alpha}}u(y)| |x-y|^{|\boldsymbol{\alpha}|} 
			+  \sum_{\boldsymbol{\alpha} \in \partial\boldsymbol{m}} \frac{|\boldsymbol{\alpha}|}{\boldsymbol{\alpha}!} |x - y|^{|\boldsymbol{\alpha}|}  
			\int_0^1 (1 - t)^{|\boldsymbol{\alpha}|-1} |D^{\boldsymbol{\alpha}} u (t x + (1-t) y)| \, dt.
		\end{split}
	\end{align*}
	If the volume of $C$ is $c \rho^d$, then the volume of $C_{X,r}$ is $c r^d$. 
	By integrating $y$ over $C_{x,r}$, we have
	\begin{align*}
			c r^d |u(x)|
			\leq \sum_{\boldsymbol{\alpha} \in \mathrm{int}\, \boldsymbol{m}} \frac{r^{|\boldsymbol{\alpha}|}}{\boldsymbol{\alpha}!}       \int_{C_{x,r}} |D^{\boldsymbol{\alpha}}u(y)| \, dy 
			+  \sum_{\boldsymbol{\alpha} \in \partial\boldsymbol{m}} \int_{C_{x,r}} |x - y|^{|\boldsymbol{\alpha}|} \, dy 
			\int_0^1 (1 - t)^{|\boldsymbol{\alpha}|-1} |D^{\boldsymbol{\alpha}} u (t x + (1-t) y)| \, dt.
	\end{align*}
	As for the final (double) integral, we begin by changing the order of integration, and then substitute $z = t x + (1 - t) y$, so that $z - x = (1 - t) (y - x)$ and $dz = (1 - t)^d \, dy$, then we obtain the integral,
	\begin{align*}
		\int_0^1 (1 - t)^{- d - 1} \, dt 
		\int_{C_{x,(1-t)r}} |z - x|^{|\boldsymbol{\alpha}|} |D^{\alpha} u(z)| \, dz.
	\end{align*}
	Changing the order of the above integration gives 
	\begin{align*}
			\int_{C_{x,t}} |x - z|^{|\boldsymbol{\alpha}|} |D^{\boldsymbol{\alpha}}u(z)| \, dz
			\int_0^{1-(|z-x|/r)} (1 - t)^{- d - 1} \, dt
			\leq \frac{r^d}{d} \int_{C_{x,r}} |x - z|^{|\boldsymbol{\alpha}|-d} |D^{\alpha} u(z)| \, dz.
	\end{align*}
	Inequality \eqref{localEs} now follows immediately.
\end{proof}

\begin{proof}[Proof of Case A of Theorem \ref{ASIT}]
	As noted earlier, we can assume that $j = 0$. 
	Let $u \in W^{\boldsymbol{m}}_p(\Omega) \cap C^{\infty}(\Omega)$ and let $x \in \Omega$. 
	We are supposed to show that
	\begin{align} \label{lehh}
	|u(x)| \leq K \|u\|_{\boldsymbol{m},p}.
	\end{align}
	%For $p = 1$ and $m = n$, this follows immediately from (8). 
	For $p > 1$ and $\underline{\boldsymbol{m}} p > d$, we apply H\"{o}lder's inequality to \eqref{localEs} with $r = \rho$ to obtain
	\begin{align*}
			|u(x)|
			\leq K \biggl( \sum_{\boldsymbol{\alpha} \in \mathrm{int}\, \boldsymbol{m} } c^{1/p'} \rho^{|\boldsymbol{\alpha}| - (d/p)} 
			\|D^{\boldsymbol{\alpha}}u\|_{p,C_{x,\rho}}
			+ \sum_{\boldsymbol{\alpha} \in \partial \boldsymbol{m}} \|D^{\boldsymbol{\alpha}}u\|_{p,C_{x,\rho}} 
			\biggl[ \int_{C_{x,\rho}} |x - y|^{(|\boldsymbol{\alpha}|-d)p'} \, dy \biggr]^{1/p'} \biggr),
	\end{align*}
	where $c$ is the volume of $C_{x,1}$ and $p' = p/(p-1)$. 
	The last integral is finite since $(|\boldsymbol{\alpha}|- d)p > (\underline{\boldsymbol{m}} - d) p' > - d$ when $\underline{\boldsymbol{m}} p > d$. Therefore
	\begin{align} \label{10}
	|u(x)| \leq K \sum_{\boldsymbol{\alpha} \in \mathrm{dom}\, \boldsymbol{m}} \|D^{\boldsymbol{\alpha}}u\|_{p,C_{x,\rho}}
	\end{align}
	and \eqref{lehh} follows since $C_{x,\, \rho} \subset \Omega$.

	Since any $u \in W^{\boldsymbol{m}}_p(\Omega)$ is the limit of a Cauchy sequence of continuous functions, and \eqref{lehh} implies this Cauchy sequence converges to a continuous function on $\Omega$, $u$ must coincide with a continuous function a.e.~on $\Omega$. 
	Therefore, we prove that $u \in C_B^0(\Omega)$ and imbedding \eqref{01} holds.
	
	Let the intersection of $\Omega$ with a $k$-dimensional plane $H$ be denoted by $\Omega_k$, $\Omega_{k,\rho} = \{ x \in \mathbb{R}^d : \mathrm{dist}(x, \Omega_k) < \rho \}$, and $u$ and all its derivatives be extended to be zero outside $\Omega$. 
	Considering that $C_{x,\, \rho} \subset B_{\rho}(x)$ where $B_{\rho}(x)$ denotes the ball of radius $\rho$ with centre at $x$, with \eqref{10} and denoting by $dx'$ the $k$-volume element in $H$, we have
\begin{align*}
\int_{\Omega_k} |u(x)|^p \, dx'
& \leq K \sum_{\boldsymbol{\alpha} \in \mathrm{dom}\, \boldsymbol{m}} \int_{\Omega_k} \, dx'
\int_{B_{\rho}(x)} |D^{\boldsymbol{\alpha}} u(y)|^p \, dy
\\
& = K \sum_{\boldsymbol{\alpha} \in \mathrm{dom}\, \boldsymbol{m}} \int_{\Omega_{k,\rho}} |D^{\boldsymbol{\alpha}}u(y)|^p \, dy
\int_{H \cap B_{\rho}(y)} \, dx'
 \leq K_1 \|u\|_{\boldsymbol{m},p,\Omega}^p,
\end{align*}
	and $W^{\boldsymbol{m}}_p(\Omega) \to L^p(\Omega_k)$. 
	However, \eqref{lehh} shows that $W^{\boldsymbol{m}}_p(\Omega) \to L^{\infty}(\Omega_k)$ 
	and hence imbedding \eqref{2} follows by Theorem 2.11 (Interpolation inequality in \cite{adams2003sobolev}).
\end{proof}

Let $\chi_r$ be the characteristic function of the ball 
$B_r(0) = \{ x \in \mathbb{R}^d : |x| < r \}$. 
In the following discussion we will develop estimates for convolutions of $L^p$ functions with the kernels $\omega_m(x) = |x|^{m-d}$ and
\begin{align*}
	\chi_r \omega_m(x) = 
	\begin{cases}
		|x|^{m-d} & \text{ if }  |x| < r, \\
		0 & \text{ if } |x| \geq r.
	\end{cases}
\end{align*}
Observe that if $m \leq d$ and $0 < r \leq 1$, then
we have 
	$\chi_r(x) 
	\leq \chi_r \omega_m(x)
	\leq \omega_m(x)$.

\begin{lemma} \label{lemma3}
	Let $p \geq 1$, $1 \leq k \leq d$, and $d - m p < k$. 
	There exists a constant $K$ such that for every $r > 0$, every $k$-dimensional plane 
	$H \subset \mathbb{R}^d$, and every $v \in L^p(\mathbb{R}^d)$, we have 
	$\chi_r \omega_m * |v| \in L^p(H)$ and
	\begin{align} \label{11}
	\|\chi_r \omega_m * |v|\|_{p,H}
	\leq K r^{m-(d-k)/p} \|v\|_{p,\mathbb{R}^d}. 
	\end{align}
	In particular,
	\begin{align*}
		\|\chi_1 * |v|\|_{p,H}
		\leq \|\chi_1 \omega_m * |v|\|_{p,H}
		\leq K \|v\|_{p,\mathbb{R}^d}.
	\end{align*}
\end{lemma}

\begin{proof}[Proof of Lemma \ref{lemma3}]
	If $p > 1$, then by H\"{o}lder's inequality
	\begin{align*}
			\chi_r \omega_m * |v|(x)
			& = \int_{B_r(x)} |v(y)| |x - y|^{-s} |x - y|^{s+m-d} \, dy
			\\
			& \leq \biggl( \int_{B_r(x)} |v(y)|^p |x - y|^{- s p} \, dy \biggr)^{1/p} 
			\biggl( \int_{B_r(x)} |x - y|^{(s+m-d)p'} \, dy \biggr)^{1/p'}
			\\
			& = K r^{s+m-(d/p)} \biggl( \int_{B_r(x)} |v(y)|^p |x - y|^{-sp} \, dy \biggr)^{1/p},
	\end{align*}
	provided $s + m - (d/p) > 0$. If $p = 1$ the same estimate holds provided $s + m - d \geq 0$ without using H\"{o}lder's inequality.
	Integrating the $p$th power of the above estimate over $H$ 
	(with volume element $dx'$), we obtain
	\begin{align*}
			\|\chi_r \omega_m * |v|\|_{p,H}^p
			& = \int_H \bigl| \chi_r \omega_m * |v|(x) \bigr|^p \, dx'
		      \leq K r^{(s+m)p-d} \int_H \, dx \int_{B_r(x)} |v(y)|^p |x - y|^{-sp} \, dy
			\nonumber\\
			& \leq K r^{(s+m)p-d} r^{k-sp} \|v\|_{p,\mathbb{R}^d}^p 
              = K r^{mp-(n-k)} \|v\|_{p,\mathbb{R}^d}^p, 
	\end{align*}
	provided $k > s p$.
	Since that for $d - m p < k$, there exists $s$ satisfying $(d/p) - m < s < k/p$, both estimates above are valid and \eqref{11} holds.
\end{proof}

\begin{lemma} \label{lemma4}
	Let $p > 1$, $m p < d$, $d - m p < k \leq d$, and $p^* = k p / (d - mp)$. 
	There exists a constant $K$ such that for every $k$-dimensional plane $H$ in $\mathbb{R}^d$ and every $v \in L^p(\mathbb{R}^d)$, we have $\omega_m * |v| \in L^{p^*}(H)$ and
	\begin{align} \label{12}
	\|\chi_1 * |v|\|_{p^*,H} 
	\leq \|\chi_1 \omega_m * |v|\|_{p^*,H}
	\leq \|\omega_m * |v|\|_{p^*,H}
	\leq K \|v\|_{p,\mathbb{R}^d}.
	\end{align}
\end{lemma}

\begin{proof}[Proof of Lemma \ref{lemma4}]
	Only the final inequality of \eqref{12} requires proof. 
	Since $m p < d$, for each $x \in \mathbb{R}^d$, H\"{o}lder's inequality gives
	\begin{align*}
			\int_{\mathbb{R}^d - B_r(x)} |v(y)| |x - y|^{m-d} \, dy
			& \leq \|v\|_{p,\mathbb{R}^d} 
			\biggl( \int_{\mathbb{R}^d - B_r(x)} |x - y|^{(m-d)p'} \, dy \biggr)^{1/p'}
			\\
			& = K_1 \|v\|_{p,\mathbb{R}^d} 
			\biggl( \int_r^{\infty} t^{(m-d)p'+d-1} \, dt \biggr)^{1/p'}
			 = K_1 r^{m-(d/p)} \|v\|_{p,\mathbb{R}^d}.
	\end{align*}
	If $t > 0$, choose $r$ so that $K_1 r^{m-(d/p)} \|v\|_{p,\mathbb{R}^d} = t/2$. If
	\begin{align*}
		\omega_m * |v|(x)
		= \int_{\mathbb{R}^d} |v(y)| |x - y|^{m-d} \, dy
		> t,
	\end{align*}
	then
	\begin{align*}
		\chi_r \omega_m * |v|(x)
		= \int_{B_r(x)} |v(y)| |x - y|^{m-d} \, dy
		> t/2.
	\end{align*}
	Thus
\begin{align*}
\mu_k \bigl( \{ x \in H : \omega_m * |v|(x) > t \} \bigr)
& \leq \mu_k \bigl( \{ x \in H : \chi_r \omega_m * |v|(x) > t/2 \} \bigr)
  \leq (2/t)^p \|\chi_r \omega_m * |v|\|_{p,H}^p
\\
& \leq \biggl( \frac{r^{(d/p)-m}}{K_1 \|v\|_{p,\mathbb{R}^d}} \biggr)^p
		K r^{mp-d+k} \|v\|_{p,\mathbb{R}^d}^p = K_2 r^k,
\end{align*}
by inequality \eqref{11}. But $r^k = (2 K_1 \|v\|_{p,\mathbb{R}^d} / t)^{p^*}$, so
	\begin{align*}
		\mu_k \bigl( \{ x \in H : \omega_m * |v|(x) > t \} \bigr)
		\leq K_2 \biggl( \frac{2 K_1}{t} \|v\|_{p,\mathbb{R}^d} \biggr)^{p^*}.
	\end{align*}
	Therefore, the mapping $I : v \mapsto (\omega_m * |v|) |_H$ is of weak type $(p, p^*)$.
	
	For fixed $m$, $d$, $k$, the values of $p$ satisfying the conditions of this lemma constitute an open interval, hence there exist $p_1$ and $p_2$ in that interval, and a number $\theta$ satisfying $0 < \theta < 1$ such that
	\begin{align*}
		\frac{1}{p} = \frac{1-\theta}{p_1} + \frac{\theta}{p_2},
	\qquad
	\text{ and }
	\qquad
		\frac{1}{p^*}
		= \frac{d/k}{p} - \frac{m}{k}
		= \frac{1-\theta}{p_1^*} + \frac{\theta}{p_2^*}.
	\end{align*}
	
	Since $p^* > p$, the Marcinkiewicz interpolation theorem 2.58 in \cite{adams2003sobolev} ensures us that $I$ is bounded from $L^p(\mathbb{R}^d)$ into $L^{p^*}(H)$, i.e.~\eqref{12} holds. 
\end{proof}

\begin{proof}[Proof of Case C of Theorem \ref{ASIT} for $p > 1$]
We have $m p < \overline{\boldsymbol{m}} p< d$, $d - m p < d - \underline{\boldsymbol{m}} p < k \leq d$, and $p \leq q \leq p^* = k p / (d - \underline{\boldsymbol{m}} p)$. 
	Let $u \in C^{\infty}(\Omega)$ and extend $u$ and all its derivatives to be zero on $\mathbb{R}^d - \Omega$. Taking $r = \rho$ in Lemma 4.15 in \cite{adams2003sobolev} and replacing $C_{x,r}$ with the larger ball $B_1(x)$, we have
	\begin{align} \label{13}
	|u(x)| 
	\leq K \biggl( \sum_{\boldsymbol{\alpha} \in \mathrm{int}\, \boldsymbol{m}}  \chi_1 * |D^{\boldsymbol{\alpha}}u|(x) 
	+ \sum_{\boldsymbol{\alpha} \in \partial \boldsymbol{m}}  \chi_1 \omega_{|\boldsymbol{\alpha}|} * |D^{\boldsymbol{\alpha}}u|(x) \biggr).
	\end{align}
	If $1/q = \theta/p + (1 - \theta)/p^*$ where $0 \leq \theta \leq 1$, then by the interpolation inequality of Theorem 2.11 in \cite{adams2003sobolev} and Lemmas \ref{lemma3} and \ref{lemma4}, 
	\begin{align*}
			\|u\|_{q,\Omega_k}
			\leq \|u||_{p,H}^{\theta} \|u\|_{p^*,H}^{1-\theta}
			\leq K \biggl( \sum_{\boldsymbol{\alpha} \in \mathrm{int}\, \boldsymbol{m}} \|D^{\boldsymbol{\alpha}}u\|_{p,\mathbb{R}^d} \biggr)^{\theta}
			\biggl( \sum_{\boldsymbol{\alpha} \in \partial\boldsymbol{m}} \|D^{\boldsymbol{\alpha}}u\|_{p,\mathbb{R}^d} \biggr)^{1-\theta}
			\leq K \|u\|_{\boldsymbol{m},p,\Omega},
	\end{align*}
	as required. 
\end{proof}

\begin{proof}[Proof of Case B of Theorem \ref{ASIT} for $p > 1$]
We have $\overline{\boldsymbol{m}} p \geq d$ and $\underline{\boldsymbol{m}} p \leq d$, $1 \leq k \leq d$, and $p \leq q < \infty$. 
We can select numbers $p_1$, $p_2$, and $\theta$ such that 
$1 < p_1 < p < p_2$, $d - \underline{\boldsymbol{m}} p_1 < k$, $0 < \theta < 1$, and
\begin{align*}
	\frac{1}{p} = \frac{\theta}{p_1} + \frac{1-\theta}{p_2},
	\qquad \qquad 
	\frac{1}{q} = \frac{\theta}{p_1}.
\end{align*}
As in the above proof of Case C for $p > 1$, the maps $v \mapsto (\chi_1 * |v|) |_H$ and $v \mapsto (\chi_1 \omega_{|\boldsymbol{\alpha}|} * |v|) |_H$ are bounded from $L^{p_1}(\mathbb{R}^d)$ into $L^{p_1}(\mathbb{R}^k)$ and so are of weak type $(p_1, p_1)$. 
As in the proof of Case A, these same maps are bounded from $L^{p_2}(\mathbb{R}^d)$ into $L^{\infty}(\mathbb{R}^k)$ and so are of weak type $(p_2, \infty)$. 
By the Marcinkiewicz theorem again, they are bounded from $L^p(\mathbb{R}^d)$ into $L^q(\mathbb{R}^k)$ and
\begin{align*}
	\|\chi_1 * |v|\|_{q,H}
	\leq \|\chi_1 \omega_{|\boldsymbol{\alpha}|}* |v|\|_{q,H}
	\leq K \|v\|_{p,\mathbb{R}^d}
\end{align*}
and the desired result follows by applying these estimates to the various terms of \eqref{13}.
\end{proof}

Here, the next Theorem \ref{Th2} gives an estimate of the $D^{\boldsymbol{\alpha}}u$ where $u \in W^{\boldsymbol{m}}_p (\Omega)$ and $\boldsymbol{\alpha} \in \mathrm{dom}\, \boldsymbol{m}$.

\begin{theorem} \label{Th2}
	Let $\Omega$ be a domain in $\mathbb{R}^{d}$ satisfying the cone condition. For each $\delta_{0}>0$ there exist finite constant $K$, each depending on $d, m, p, \delta_{0}$ and the dimensions of the cone $C$ providing the cone condition for $\Omega$ such that if $0<\delta\leq\delta_{0},0\leq j\leq m$, and $u\in W^{\boldsymbol{m}}_p(\Omega)$ , then
	\begin{align}
	\|D^{\boldsymbol{\alpha}}u\|_{p} \leq K\delta^{\underline{\boldsymbol{m}}- |\boldsymbol{\alpha}|}\|u\|_{\boldsymbol{m},p}+K\delta^{-|\boldsymbol{\alpha}|}\|u\|_{p}.
	\end{align}
\end{theorem}

	In order to prove this theorem, we need the following three Lemmas \ref{lemma2-1}, \ref{lemma2-2} and \ref{lemma 2-3}. 
	\begin{lemma} \label{lemma2-1}
		If $\rho>0,1\leq p<\infty, K_{p}=2^{p-1}9^{p}$, and $g\in C^{2}([0, \rho])$ then
		\begin{align} \label{4}
		|g'(0)|^{p}\displaystyle \leq\frac{K_{p}}{\rho}\bigg(\rho^{p}\int_{0}^{\rho}|g''(t)|^{p}dt+\rho^{-p}\int_{0}^{\rho}|g(t)|^{p}dt\bigg). 
		\end{align}
	\end{lemma}
	\begin{proof}[Proof of Lemma \ref{lemma2-1}]
		Let $f\in C^{2}([0,1])$, $x\in[0,1/3]$, and $y\in[2/3,1]$. According to the mean-value theorem, there exists $z\in(x,\ y)$ such that
		$$
		|f'(z)|=\bigg|\frac{f(y)-f(x)}{y-x}\bigg|\leq 3|f(x)|+3|f(y)|.
		$$
		Therefore,
		\begin{align*}
		|f'(0)| 
		= \bigg|f'(z)-\int_{0}^{z}f''(t)dt\bigg|
		\leq 3|f(x)|+3|f(y)|+\int_{0}^{1}|f''(t)|dt.
		\end{align*}
		Integrating $x$ over $[0,1/3]$ and $y$ over [2/3, 1] yields
		$$
		\frac{1}{9}|f'(0)|\leq\int_{0}^{1/3}|f(x)|dx+\int_{2/3}^{1}|f(y)|dy+\frac{1}{9}\int_{0}^{1}|f''(t)|dt.
		$$
		For $p\geq 1$, we then have (using H\"{o}lder's inequality if $p>1$)
		$$
		|f'(0)|^{p}\leq K_{p}\bigg(\int_{0}^{1}|f''(t)|^{p}dt+\int_{0}^{1}|f(t)|^{p}dt\bigg)
		$$
		where $K_{p}=2^{p-1}9^{p}$. Inequality \eqref{4} now follows by substituting $f(t)=g(\rho t)$ . 
	\end{proof}
	
	\begin{lemma} \label{lemma2-2}
		If $ 1\leq p<\infty$ and the domain $\Omega\subset \mathbb{R}^{n}$ satisfies the cone condition, then there exists a constant $K$ depending on $d, p$, and the height $\rho_{0}$ and aperture angle $\kappa$ of the cone $C$ providing the cone condition for $\Omega$ such that for all $\epsilon, 0<\epsilon\leq p_{0}$ and all $u\in W^{\boldsymbol{m}}_p(\Omega)$ with $m_j\geq 2$, we have
		\begin{align} \label{5}
		\|D_j \, u(x) \| \leq  K_p \big(\epsilon\|D^2_j \, u(x)\|_{p}+\epsilon^{-1}||u\Vert_{p}^{p} \big)
		\end{align}
	\end{lemma}
	\begin{proof}[Proof of Lemma \ref{lemma2-2}]
		Let $\Sigma=\{\sigma \in \mathbb{R}^{d}\ :\ |\sigma|=1\}$ be the unit sphere in $\mathbb{R}^{d}$ with volume element $ d\sigma$ and $(d-1)$-volume $ K_{0}=K_{0}(d)=\int_{\Sigma}d\sigma$. If $ x\in\Omega$ let $\sigma_{x}$ be the unit vector in the direction of the axis of a cone $ C_{x}\subset\Omega$ congruent to $C$ and having vertex at $x$, and let $\Sigma_{x}=\{\sigma\in\Sigma\ :\ \angle(\sigma,\ \sigma_{x})\leq\kappa/2\}.$
		Let $u\in C^{\infty}(\Omega)$ . If $x\in\Omega, \sigma\in\Sigma_{x}$, and $0<p\leq p_{0}$, then
		$$|\sigma\cdot \mathrm{grad}\, u (x)|^{p}\leq\frac{K_{p}}{\rho}I(\rho,p,u,x,\sigma),
		$$
		where
		$$
		I(\rho,p, u, x,\sigma)=\rho^{p}\int_{0}^{p}|D_{t}^{2}u(x+t \sigma)|^{p}d t+\rho^{-p}\int_{0}^{p}|u(x+t\sigma)|^{p}dt.
		$$

%		There exists a constant $K_{1}=K_{1}(d,\ p,\ \kappa)$ such that
%		$$
%		\int_{\Sigma}|\sigma\cdot \mathrm{grad}\, u(x)|^{p}d  \sigma\geq\int_{\Sigma_{x}}|\sigma\cdot 
%		\mathrm{grad}\, u(x)|^{p}d\sigma\geq K_{1}| \mathrm{grad}\, u(x)|^{p}.
%		$$
%		Accordingly,
%		$$
%		\int_{\Omega}|\mathrm{grad}\, u(x)|^{p}dx \leq\frac{K_{p}}{K_{1}p}\int_{\Sigma}d\sigma\int_{\Omega} I(\rho,p, u, x,\sigma)dx.
%		$$

		In order to estimate the inner integral on the right side, regard $u$ and its derivatives as being extended to all of $\mathbb{R}^{d}$ so as to be identically zero outside $\Omega$. For simplicity, we suppose $\sigma=e_{d}=(0,\ \ldots,\ 0,1)$ and write $x=(x',\, x_{d})$ with $x'\in \mathbb{R}^{d-1}$, we have
		\begin{align*}
			\int_{\Omega}	I(\rho,p, u, x,e_n)dx
			&= \int_{\mathbb{R}^{d-1}}\,dx'\int_{-\infty}^{\infty}dx_{n}\int_{0}^{\rho}\bigg(\rho^{p}\big|D_{d}^{2}u(x', x_{d}+t)\big|^{p}+\rho^{-p}\big|u(x', x_{d}+t)\big|^{p}\bigg)dt
			\\
			&= \int_{\mathbb{R}^{d-1}}\,dx'\int_{0}^{\rho}dt\int_{-\infty}^{\infty}\bigg(\rho^{p}\big|D_{n}^{2}u(x)\big|^{p}+\rho^{-p}\big|u(x)\big|^{p}\bigg)dx_{d}
			\\
			&\leq  \rho \int_{\Omega}\bigg(\rho^{p} \big|D_{d}^{2}u(x)\big |^{p}+\rho^{-p} \big|u(x) \big|^{p}\bigg)dx.
		\end{align*}
		If we take $\sigma = e_j$,$j = 1,\cdots, d$, then we can get 
		$$
		\int_{\Omega} \big|e_j \cdot \mathrm{grad}\, u(x)\big|^{p}\leq\frac{K_{p}}{p}\int_{\Omega} I(\rho,p,u,x,e_j)dx,
		$$
		in other words,
		\begin{align*}
		\|D_j \, u(x) \|_p \leq  K_p \big(\rho^{p}\|D^2_j \, u(x)\|^p_{p}+\rho^{-p}||u\Vert_{p}^{p} \big).
		\end{align*}
		Inequality \eqref{5} now follows by taking $p$th roots, replacing $\rho$ with $\epsilon$, and noting that $C^{\infty}(\Omega)$ is dense in $W^{\boldsymbol{m}}_p(\Omega)$ .
	\end{proof}
	
	\begin{lemma} \label{lemma 2-3}
		Let $|\boldsymbol{m}|\geq 2$, let $ 0<\delta_{0}<\infty$, and let $\displaystyle \epsilon_{0}=\min\{\delta_{0},\ \delta_{0}^{2},\ \ldots,\ \delta_{0}^{\overline{\boldsymbol{m}}-1}\}.$ Suppose that for given $p,  1\leq p<\infty$, and given $\Omega\subseteq \mathbb{R}^{d}$ there exists a constant 
		$K=K(\epsilon_{0},\boldsymbol{m},\ p,\ \Omega)$ such that for every $\epsilon$ satisfying $0<\epsilon\leq\epsilon_0$, every $\boldsymbol{\alpha}$ satisfying $\boldsymbol{\alpha}\in \mathrm{dom}\,\boldsymbol{m}$, and every $u\in W^{\boldsymbol{m}}_p(\Omega)$ , we have
		$\exists\, \boldsymbol{\ell} \in \partial\boldsymbol{m}$,
		\begin{align}\label{7}
		\|D^{\boldsymbol{\alpha}}u\|_{p}\leq K\epsilon^{|\boldsymbol{\ell}| - |\boldsymbol{\alpha}| }\|D^{\boldsymbol{\alpha}}u\|+K\epsilon^{-|\boldsymbol{\alpha}| }\|u\|_p
		\end{align}
	\end{lemma}
	
	\begin{proof}[Proof of Lemma \ref{lemma 2-3}]
		Since $\eqref{7}$ is trivial for $\boldsymbol{\alpha} = 0$, so we consider only the case $|\boldsymbol{\alpha}| \geq 1$. 
		The proof is accomplished by a double induction on $|\boldsymbol{\alpha}|$ and $ |\boldsymbol{\ell} -  \boldsymbol{\alpha}|$ with ${\ell}_i \geq {\alpha}_i$ $i =1, \ldots, d$. The constants $K_{1}, K_{2}, \ldots$ appearing in the argument may depend on $\delta_{0}$ (or $\epsilon_{0}$), $\boldsymbol{m}, p$, and $\Omega.$ 
		
		First we prove \eqref{7} for $|\boldsymbol{\ell} -  \boldsymbol{\alpha}| =1 $, i.e.  $\boldsymbol{\ell} -  \boldsymbol{\alpha} = e_j $ by induction on $|\boldsymbol{\alpha}|$, so that \eqref{5} is the special case $|\boldsymbol{\alpha}|= 1$. Assume, therefore, that for some $k, 1\leq |\boldsymbol{\beta}|\leq k-2,$
		\begin{align} \label{8}
			\|D^{\boldsymbol{\beta}}u\|_{p} \leq K_{1}\delta\|D^{\boldsymbol{\beta}+e_j}u\|_{p}+K_{1}\delta^{-|\boldsymbol{\beta}|}\|u\|_{p}
		\end{align}
		holds for all $\delta,\, 0<\delta\leq\delta_{0}$. Now we prove the above equation with  $|\boldsymbol{\alpha}|=k-1$. We obtain from the Lemma \ref{lemma2-2},
		$$
		\|D^{\boldsymbol{\alpha} }u\|_{p} = \|D_j D^{\boldsymbol{\alpha} - e_j}u\|_{p}\leq K_{p}\delta\|D_j^2D^{\boldsymbol{\alpha} + e_j}u\|_{p}+K_{p}\delta^{-1}\|D^{\boldsymbol{\alpha} + e_j}u\|_{p}.
		$$
		Combining this inequality with \eqref{8}, we obtain, for $0<\eta\leq\delta_{0},$
		$$
		\|D^{\boldsymbol{\alpha} }u\|_{p} \leq  K_p \delta \|D^{\boldsymbol{\alpha} + e_j}u\|_{p} + K_p K_{1}\delta^{-1}\eta\|D^{\boldsymbol{\beta}+e_j}u\|_{p}+K_p K_{1}\delta^{-1}\eta^{-|\boldsymbol{\beta}|}\|u\|_{p}
		$$
		We may assume without prejudice that $2K_{1}K_{p}\geq 1$. Therefore, we may take $\eta=\delta/(2K_{1}K_{p})$ and so obtain
		\begin{align*}
			\|D^{\boldsymbol{\alpha} }u\|_{p} 
			\leq  2K_p \delta \|D^{\boldsymbol{\alpha} + e_j}u\|_{p} + (\frac{\delta}{2K_1 K_p})^{-|\boldsymbol{\alpha}|}\|u\|_p
		    \leq  K_2 \delta \|D^{\boldsymbol{\alpha} + e_j}u\|_{p} + K_2\delta^{-|\boldsymbol{\alpha}|}\|u\|_p.
		\end{align*}
		This completes the induction establishing in \eqref{8} for $0<\delta\leq\delta_{0}$, and hence \eqref{7} for $| \boldsymbol{\ell} - \boldsymbol{\alpha}|= 1$ and $0<\epsilon\leq\delta_{0}.$
		
		We now prove by induction on $|\boldsymbol{\ell} - \boldsymbol{\alpha}|$ that
		\begin{align} \label{9}
		\|D^{\boldsymbol{\alpha}}u\|_{p}\leq K_{3}\delta^{|\boldsymbol{\ell}|- |\boldsymbol{\alpha}|}\|D^{\boldsymbol{\ell}}u\|_{p}+K_{3}\delta^{-|\boldsymbol{\alpha}|}\|u\|_{p}
		\end{align}
		holds for $1\leq  |\boldsymbol{\ell} - \boldsymbol{\alpha}| \leq m-1$ with $\ell_i \geq {\alpha}_i$ and $0<\delta\leq\delta_{0}$. Note that \eqref{8} is the special case $|\boldsymbol{\ell} - \boldsymbol{\alpha}|= 1$ of \eqref{9}. Assume, therefore, that \eqref{9} holds for some $\boldsymbol{\ell}, 1\leq |\boldsymbol{\ell} - \boldsymbol{\alpha}| \leq k-1$. We prove that it also holds for $\boldsymbol{\ell}$ with $|\boldsymbol{\ell} - \boldsymbol{\alpha}| = k$. From \eqref{8} and \eqref{9}, we obtain
		\begin{align*}
			\|D^{\boldsymbol{\alpha}}u\|_{p} 
			&\leq K_{4}\delta\|D^{\boldsymbol{\alpha}+e_j}u\|_{p}+K_{4}\delta^{-|\boldsymbol{\alpha}|}\|u\|_{p}
			\\
			&\leq K_{4}\delta(K_3\delta^{|\boldsymbol{\ell}|- |\boldsymbol{\alpha}|-1}\|D^{\boldsymbol{\ell}}u\|_{p}+K_{3}\delta^{-|\boldsymbol{\alpha}|-1}\|u\|_{p})+K_{4}\delta^{-|\boldsymbol{\alpha}|}\|u\|_{p}
			\leq K_5\delta^{|\boldsymbol{\ell}|- |\boldsymbol{\alpha}|}\|D^{\boldsymbol{\ell}}u\|_{p}+K_{5}\delta^{-|\boldsymbol{\alpha}|}\|u\|_{p}.
		\end{align*}
		Therefore, \eqref{9} holds for $|\boldsymbol{\ell} - \boldsymbol{\alpha}| = k$, and \eqref{7} follows by setting $\delta=\epsilon$ in \eqref{7} and noting that $\epsilon\leq\epsilon_{0}$ if $\delta\leq\delta_{0}$. 
	\end{proof}
	
\begin{proof}[Proof of Theorem \ref{Th2}]
	It is not hard to find that $\forall\boldsymbol{\alpha}\in \partial \boldsymbol{m}$, $|\boldsymbol{\alpha}| \geq \underline{\boldsymbol{m}}$. So when $\boldsymbol{\alpha}\notin \partial \boldsymbol{m}$, we can find a $\boldsymbol{\ell}\in \partial \boldsymbol{m}$, such that
	\begin{align*} 
	\|D^{\boldsymbol{\alpha}}u\|_{p} 
	\leq K_6\delta^{|\boldsymbol{\ell}|- |\boldsymbol{\alpha}|}\|D^{\boldsymbol{\ell}}u\|_{p}+K_{6}\delta^{-|\boldsymbol{\alpha}|}\|u\|_{p}
	\leq K_6\delta^{\underline{\boldsymbol{m}}- |\boldsymbol{\alpha}|}\|u\|_{\boldsymbol{m},p}+K_{6}\delta^{-|\boldsymbol{\alpha}|}\|u\|_{p}.
	\end{align*}
This completes the proof of Theorem \ref{Th2}.	
\end{proof}

The following interpolation Theorem \ref{Th1} provides sharp estimates for $L^q$ norms of functions in $W^{\boldsymbol{m}}_p(\Omega)$.
\begin{theorem} \label{Th1}
	Let $\Omega$ be a domain in $\mathbb{R}^{d}$ satisfying the cone condition. If $\underline{\boldsymbol{m}} p > d$, let $ p\leq q\leq\infty$;
	if $\overline{\boldsymbol{m}} p < d$, let $p\leq q \leq p^{*}=dp/(d - \overline{\boldsymbol{m}}p)$ . Then there exists a constant $K$ depending on $\boldsymbol{m}, d, p, q$ and the dimensions of the cone $C$ providing the cone condition for $\Omega,$ such that for all $u\in W^{\boldsymbol{m}}_p(\Omega)$ ,
	\begin{center}
		$\Vert u\Vert_{q}\leq K\Vert u\Vert_{\boldsymbol{m},p}^{\theta}\Vert u\Vert_{p}^{1-\theta}$,  
	\end{center}
	where $\theta = d/\underline{\boldsymbol{m}}p - d/\underline{\boldsymbol{m}}q$ .
\end{theorem}

\begin{proof}[Proof of Theorem \ref{Th1}]
	The case $\overline{\boldsymbol{m}} p< d, p\leq q\leq p^{*}$ follows directly from Theorem 2.11 in \cite{adams2003sobolev} and Theorem \ref{ASIT}:
	$$
	\Vert u\Vert_{q}\leq||u\Vert_{p^{*}}^{\theta}\Vert u\Vert_{p}^{1-\theta}\leq \|u\Vert_{\boldsymbol{m},p}^{\theta}\Vert u\Vert_{p}^{1-\theta},
	$$
	where $1/q=(\theta/p^{*})+(1-\theta)/p$ from which it follows that $\theta=(d/\underline{\boldsymbol{m}} p)-(d/\underline{\boldsymbol{m}}q)$ .  
	
	For the cases $\overline{\boldsymbol{m}} p \geq d$ and $\underline{\boldsymbol{m}} p \leq d$, $ p\leq q<\infty$, and $\underline{\boldsymbol{m}} p > d,  p\leq q\leq\infty$ we use the local bound obtained in Lemma \ref{ALE}. If $ 0<r\leq\rho$ (the height of the cone $C$), then
	\begin{align} \label{11-11}
	|u(x)| \leq K_{1}\biggl(\sum_{ \boldsymbol{\alpha} \in \mathrm{int}\, \boldsymbol{m} }r^{|\alpha|-n}\chi_{r}*|D^{\alpha}u|(x)+\sum_{\boldsymbol{\alpha} \in \partial \boldsymbol{m} }\big(\chi_{r}\omega_{|\boldsymbol{\alpha}|}\big)*|D^{\alpha}u|(x)\biggr)
	\end{align}
	where $\chi_{r}$ is the characteristic function of the balI of radius $r$ centred at the origin in $\mathbb{R}^{d}$, and $\omega_{m}(x)=|x|^{m-d}$. We estimate the $L^{q}$ norms of both terms on the right side of \eqref{11-11} using Young's inequality from Corollary 2.25 in \cite{adams2003sobolev}. If $(1/p)+(1/s)=1+(1/q)$ , then
	\begin{align*}
	\big\| \chi_r * |D^{\boldsymbol{\alpha}}u| \big\|_q
	& \leq \big\| \chi_r \big\|_s  \big\| D^{\boldsymbol{\alpha}}u \big\|_p
	  = K_2 r^{d-(d/p)+(d/q)} \big\| D^{\boldsymbol{\alpha}}u \big\|_p,
	\\
	\big\| (\chi_r \omega_{|\boldsymbol{\alpha}|}) * |D^{\boldsymbol{\alpha}}u| \big\|_q 
	& \leq \big\| \chi_r \omega_{|\boldsymbol{\alpha}|} \big\|_s  
	       \big\| D^{\boldsymbol{\alpha}}u \big\|_p 
	  = K_3 r^{|\boldsymbol{\alpha}|-(d/p)+(d/q)} 
	    \big\| D^{\boldsymbol{\alpha}}u \big\|_p.
    \end{align*}
	Note that $|\boldsymbol{\alpha}|-(d/p)+(d/q)>0$ if $q$ satisfies the above restrictions. Hence
	\begin{align*}
		\|u\Vert_{q} 
		\leq  K_{4}\bigg(\sum_{ \boldsymbol{\alpha} \in \mathrm{int}\, \boldsymbol{m}} r^{|\boldsymbol{\alpha}|-(d/p)+(d/q)}||D^{\boldsymbol{\alpha}}u||_{p}+\sum_{ \boldsymbol{\alpha} \in \partial \boldsymbol{m}}r^{|\boldsymbol{\alpha}|-(d/p)+(d/q)}||D^{\boldsymbol{\alpha}}u||_{p}\bigg)
		= K_{4} \sum_{\boldsymbol{\alpha} \in \mathrm{dom}\, \boldsymbol{m}} r^{|\boldsymbol{\alpha}|-(d/p)+(d/q)}||D^{\boldsymbol{\alpha}}u||_{p}.
	\end{align*}
	By Theorem \ref{Th2}, we obtain
	$$
	\|D^{\boldsymbol{\alpha}}u\|_{p} \leq K_5 r^{\underline{\boldsymbol{m}}- |\boldsymbol{\alpha}|}\|u\|_{\boldsymbol{m},p}+K_5 r^{-|\boldsymbol{\alpha}|}\|u\|_{p},
	$$
	so that, we have
	$$
	\Vert u\Vert_{q}\leq K_{6}(r^{\underline{\boldsymbol{m}}-(d/p)+(d/q)}\Vert u\|_{\boldsymbol{m},p}+r^{-(d/p)+(d/q)}\Vert u\Vert_{p})\ .
	$$
	Adjusting $K_{6}$ if necessary, we can assume this inequality holds for all $r\leq 1.$ Choosing $r = (\frac{\| u\|_{p}}{\| u\|_{\boldsymbol{m},p}})^{1/\underline{\boldsymbol{m}}}$ to make the two terms on the right side equal, we obtain the conclusion with $\theta = d/\underline{\boldsymbol{m}}p - d/\underline{\boldsymbol{m}}q$.
\end{proof}

\begin{theorem}[Imbedding Theorem for Anisotropic Besov Spaces] \label{ITABS}
	Let $\Omega$ be a domain in $\mathbb{R}^d$ satisfying the cone condition, and let $1\leq p\leq \infty$ and $1\leq q\leq \infty$. 
	
	(a)	If $\underline{\boldsymbol{\alpha}} = d/p$, then $B^{\boldsymbol{\alpha}}_{p, 1} \to C^0_B(\Omega) \to L^{\infty}(\Omega)$.
	
	(b)	If $\underline{\boldsymbol{\alpha}} > d/p$, then $B^{\boldsymbol{\alpha}}_{p, \infty} \to C^0_B(\Omega)$.
\end{theorem}	
\begin{proof}[Proof of Theorem \ref{ITABS}]
	\textit{(a)}
	First, we choose  $\gamma\in (0,1)$ and $\boldsymbol{m}$ such that $\gamma \boldsymbol{m} = \boldsymbol{\alpha}$ and $\underline{\boldsymbol{m}} > n/p$. We take $\gamma = \theta := \frac{n}{\underline{\boldsymbol{m}}p} < 1$.
	Let $u \in B^{\boldsymbol{\alpha}}_{p,1}(\Omega) = (L_p(\Omega), W^{\boldsymbol{m}}_{p}(\Omega))_{\gamma, 1;J}$ with $\gamma \in (0,1)$. 
	
	By the discrete version of the J-method in theorem 7.15 of \cite{adams2003sobolev}, there exist functions $u_i$ in $W^{\boldsymbol{m}}_p(\Omega)$ such that the series $\sum_{i = -\infty}^{\infty} u_i $ converges to $u$ in $B^{\boldsymbol{\alpha}}_{p,1}(\Omega)$ and such that the sequence $\{2^{-i\gamma} J(2^i, u_i)\}_ {-\infty}^{\infty}$ belongs to $\ell^1$ and has $\ell^1$ norm no larger than $C\|u\|_{B^{\boldsymbol{\alpha}}_{p,1}}$. Since $\underline{\boldsymbol{m}} p > d$ and $\Omega$ satisfies the cone condition, Theorem 1 shows that 
	\begin{align*}
		\|v\|_{\infty} \leq C_1 \|v\|_p ^{1- \theta} \|v\|^{\theta}_{\boldsymbol{m}, p}
	\end{align*}
	for all $v \in  W^{\boldsymbol{m}}_{p}(\Omega)$. Thus,
	\begin{align*}
		\begin{split}
		\|u\|_{\infty} 
		\leq \sum_{-\infty}^{\infty} \|u_i\|_{\infty}
		\leq C_1 \sum_{-\infty}^{\infty} \|u_i\|_p^{1 - \theta} \|u\|^{\theta}_{\boldsymbol{m}, p}
		\leq C_1 \sum_{-\infty}^{\infty} 2^{-i\theta} J(2^i; u_i)
	    = C_1 \sum_{-\infty}^{\infty} 2^{-i\gamma} J(2^i; u_i)
		\leq C_2 \ \|u\|_{B^{\boldsymbol{\alpha}}_{p,1}}.
		\end{split}
	\end{align*}
	Since $\underline{\boldsymbol{m}} \gamma = \underline{\boldsymbol{\alpha}}$, if the condition $ \underline{\boldsymbol{m}} = d/p$ is satisfied, the conclusion holds.
	
	\textit{(b)} follows from (a).
	Since $W^{\boldsymbol{m}}_p(\Omega) \to L_p(\Omega)$, according to the Theorem \ref{ASIT}, we know that $(L_p(\Omega), W^{\boldsymbol{m}}_p(\Omega))_{\theta^{'},\infty}$ $\to (L_p(\Omega),W^{\boldsymbol{m}}_p(\Omega))_{\theta, 1}$, if $\theta^{'} > \theta$. 
	When $\theta = \frac{d}{\underline{\boldsymbol{m}} p}$, $\theta^{'} = \gamma$, we finish the proof.
\end{proof}

\section*{Reference}
\bibliographystyle{elsarticle-num} 
\bibliography{HANGbib}


%\begin{thebibliography}{00}
%%% Text of bibliographic item
%
%\bibitem{}
%
%\end{thebibliography}
\end{document}